\documentclass[journal]{IEEEtran}

\usepackage{epsfig, graphics,subfigure,wrapfig}
\usepackage{latexsym}
\usepackage{amsmath,amssymb,amsthm,amsfonts}
\usepackage{algorithm,algorithmic}
\usepackage[numbers]{natbib}

\newtheorem{lemma}{Lemma}
\newtheorem{theorem}{Theorem}
\newtheorem{proposition}{Proposition}

\newtheorem{corollary}{Corollary}

\newcommand\BlockSparseSplit[1]{B(d)}
\newcommand\SparseSplit[1]{S(d)}

\newcommand{\myoparagraph}[1]{\noindent {\it #1.}}

\newcommand{\comment}[1]{}
\def\real{\mathbb{R}}

\def\support{\text{Supp}}
\def\rowsupport{\text{RowSupp}}
\def\sgn{\text{sign}}

\newcommand{\tr}[2]{\left\langle#1,#2\right\rangle}

\def\TrueParMatrix{\bar{\Theta}}
\def\ParMatrix{\Theta}
\def\EstDual{\widetilde{Z}}
\def\delpar{\Delta}

\def\Bst{B^*}
\def\Ss{S^*}
\def\bsj{b^*_j}
\def\ssj{s^*_j}
\def\bsjk{b^{*(k)}_j}
\def\ssjk{s^{*(k)}_j}
\def\bsjkl{b_j^{*\left(k_l\right)}}
\def\ssjkl{s_j^{*\left(k_l\right)}}

\def\Bh{\hat{B}}
\def\Sh{\hat{S}}
\def\1b{\hat{b}_j}
\def\As{\hat{s}_j}
\def\2b{\hat{b}^{(k)}_j}
\def\Bs{\hat{s}^{(k)}_j}
\def\3b{\hat{b}_{j_0}^{\left(k_0\right)}}
\def\Cs{\hat{s}_{j_0}^{\left(k_0\right)}}
\def\4b{\hat{b}_{j_0}}
\def\Ds{\hat{s}_{j_0}}
\def\5b{\hat{b}_{j_0}^{(k)}}
\def\Es{\hat{s}_{j_0}^{(k)}}
\def\6b{\hat{b}_{j_0}^{\left(k^*\right)}}
\def\Fs{\hat{s}_{j_0}^{\left(k^*\right)}}
\def\7b{\hat{b}_{j_0}^{\left(k_l\right)}}
\def\Gs{\hat{s}_{j_0}^{\left(k_l\right)}}
\def\8b{\hat{b}^{(k)}}
\def\Hs{\hat{s}^{(k)}}
\newcommand\mysqrt[1]{\left(#1\right)^{1/2}}

\begin{document}

\title{A Dirty Model for Multiple Sparse Regression}

\author{
Ali Jalali, Pradeep Ravikumar\IEEEmembership{},~and~
Sujay Sanghavi,~\IEEEmembership{Member}\thanks{The authors are with
the Departments of Electrical and Computer Engineering (Jalali and Sanghavi) and Computer Science (Ravikumar), The University of Texas at Austin, Austin, TX 78712 USA email: (alij@mail.utexas.edu; pradeepr@cs.utexas.edu; sanghavi@mail.utexas.edu) }}

\maketitle

\begin{abstract}
Sparse linear regression -- finding an unknown vector from linear measurements -- is now known to be possible with fewer samples than variables, via methods like the LASSO. We consider the multiple sparse linear regression problem, where several related vectors -- with {\em partially shared} support sets --  have to be recovered. A natural question in this setting is whether one can use the sharing to further decrease the overall number of samples required. A line of recent research has studied the use of $\ell_1/\ell_q$ norm block-regularizations with $q > 1$ for such problems; however these could actually perform {\em worse} in sample complexity -- vis a vis solving each problem separately ignoring sharing -- depending on the level of sharing. 

We present a new method for multiple sparse linear regression that can  leverage support and parameter overlap when it exists, but not pay a penalty when it does not.  a very simple idea: we decompose  the parameters into two components and {\em regularize these  
differently.} We show both theoretically and empirically, our method  strictly and noticeably outperforms both $\ell_1$ or $\ell_1/\ell_q$  
methods, over the entire range of possible overlaps  (except at  boundary cases, where we match the best method). We also provide  
theoretical guarantees that the method performs well under high-dimensional scaling.
\end{abstract}


\section{Introduction: Motivation and Setup}

\myoparagraph{High-dimensional scaling} In fields across science and engineering, we are increasingly faced with problems where the number of variables or features $p$ is larger than the number of observations $n$. Under such high-dimensional scaling, for any hope of statistically consistent estimation, it becomes vital to leverage any potential structure in the problem such as sparsity (e.g. in compressed sensing~\citep{CS} and LASSO~\cite{Lasso}), low-rank structure \citep{RechtLowRank,NWLowRank}, or sparse graphical model structure \cite{Ising}. It is in such high-dimensional contexts in particular that multi-task learning~\citep{CaruanaMTL} could be most useful. Here, multiple tasks share some common structure such as sparsity, and estimating these tasks jointly by leveraging this common structure could be more statistically efficient.

\myoparagraph{Block-sparse Multiple Regression} A common multiple task learning setting, and which is the focus of this paper, is that of multiple regression, where we have $r > 1$ response variables, and a common set of $p$ features or covariates. The $r$ tasks could share certain aspects of their underlying distributions, such as common variance, but the setting we focus on in this paper is where the response variables have {\em simultaneously sparse} structure: the index set of relevant features for each task is sparse; and there is a large overlap of these relevant features across the different regression problems. Such ``simultaneous sparsity'' arises in a variety of contexts \citep{TroppSS};
indeed, most applications of sparse signal recovery in contexts ranging from graphical model learning, kernel learning, and function estimation have natural extensions to the simultaneous-sparse setting \citep{Ising,BachMKL,Spam}. \par It is useful to represent the multiple regression parameters via a matrix, where each column corresponds to a task, and each row to a feature. Having simultaneous sparse structure then corresponds to the matrix being largely ``block-sparse'' -- where each row is either all zero or mostly non-zero, and the number of non-zero rows is small. A lot of recent research in this setting has focused on $\ell_1/\ell_q$ norm regularizations, for $q > 1$, that encourage the parameter matrix to have such block-sparse structure. Particular examples include results using the $\ell_1/\ell_\infty$ norm \citep{TurlachVW, ZH08,NWJoint}, and the $\ell_1/\ell_2$ norm~\citep{Lounici09,Obozinski10}.

\myoparagraph{Our Model} Block-regularization is ``heavy-handed'' in two ways. By strictly encouraging shared-sparsity, it assumes that all relevant features are shared, and hence suffers under settings, arguably more realistic, where each task depends on features specific to itself in addition to the ones that are common. The second concern with such block-sparse regularizers is that the $\ell_1/\ell_q$ norms can be shown to encourage the entries in the non-sparse rows taking nearly identical {\em values}. Thus we are far away from the original goal of multitask learning: not only do the set of relevant features have to be exactly the same, but their values have to as well. Indeed recent research into such regularized methods \cite{NWJoint,Obozinski10} caution against the use of block-regularization in regimes where the supports and values of the parameters for each task can vary widely. Since the true parameter values are unknown, that would be a worrisome caveat.

We thus ask the question: can we learn multiple regression models by leveraging whatever overlap of features there exist, and without requiring the parameter values to be near identical? Indeed this is an instance of a more general question on whether we can estimate statistical models where the data may not fall cleanly into any one structural bracket (sparse, block-sparse and so on). With the explosion of complex and \emph{dirty} high-dimensional data in modern settings, it is vital to investigate estimation of corresponding \emph{dirty} models, which might require new approaches to biased high-dimensional estimation. In this paper we take a first step, focusing on such dirty models for a specific problem: simultaneously sparse multiple regression. 

Our approach uses a simple idea: while any one structure might not capture the data, a superposition of structural classes might.
Our method thus searches for a parameter matrix that can be \emph{decomposed} into a row-sparse matrix (corresponding to the overlapping or shared features) and an elementwise sparse matrix (corresponding to the non-shared features). As we show both theoretically and empirically, with this simple fix we are able to leverage any extent of shared features, while allowing disparities in support and values of the parameters, so that we are \emph{always} better than both the Lasso or block-sparse regularizers (at times remarkably so). 

The rest of the paper is organized as follows: In Sec 2. basic definitions and setup of the problem are presented. Main results of the paper is discussed in sec 3. Experimental results and simulations are demonstrated in Sec 4.

{\bf Notation:} For any matrix $M$, we denote its $j^{th}$ row as $m_j$, and its $k$-th column as $m^{(k)}$. The set of all non-zero rows (i.e. all rows with at least one non-zero element) is denoted by $\rowsupport(M)$ and its support by $\support(M)$. Also, for any matrix $M$, let $\|M\|_{1,1} := \sum_{j,k} |m_{j}^{(k)}|$, i.e. the sums of absolute values of the elements, and $\|M\|_{1,\infty} := \sum_j \|m_j\|_\infty$ where, $\|m_j\|_\infty := \max_k |m_{j}^{(k)}|$.

\section{Problem Set-up and Our Method}

\myoparagraph{Multiple regression} We consider the following standard multiple linear regression model: 
\begin{eqnarray*}
y^{(k)} = X^{(k)} \bar{\theta}^{(k)} + w^{(k)},\quad k = 1,\hdots,r,
\end{eqnarray*}
where, $y^{(k)} \in \real^{n}$ is the response for the $k$-th task, regressed on the design matrix $X^{(k)} \in \real^{n \times p}$ (possibly different across tasks), while $w^{(k)} \in \real^{n}$ is the noise vector. We assume each $w^{(k)}$ is drawn independently from $\mathcal{N}(0,\sigma^2)$. The total number of tasks or target variables is $r$, the number of features is $p$, while the number of samples we have for each task is $n$. For notational convenience, we collate these quantities into matrices $Y \in \real^{n\times r}$ for the responses, $\TrueParMatrix\in\real^{p\times r}$ for the regression parameters and $W \in \real^{n\times r}$ for the noise.

\myoparagraph{Our Model} In this paper we are interested in estimating the true parameter $\TrueParMatrix$ from data $\{y^{(k)},X^{(k)}\}$ by leveraging 
any (unknown) extent of simultaneous-sparsity. In particular, certain rows of $\TrueParMatrix$ would have many non-zero entries, corresponding to features shared by several tasks (``shared'' rows), while certain rows would be elementwise sparse, corresponding to those features which are relevant for some tasks but not all (``non-shared rows''), while certain rows would have all zero entries, corresponding to those features that are not relevant to any task. We are interested in estimators $\widehat{\Theta}$ that automatically adapt to different levels of sharedness, and yet enjoy the following guarantees:\\

{\bf Support recovery:} We say an estimator $\widehat{\Theta}$ successfully recovers the true signed support if
$\text{sign}(\support(\widehat{\Theta})) = \text{sign}(\support(\TrueParMatrix))$. We are interested in deriving sufficient conditions under which the estimator succeed. We note that this is stronger than merely recovering the row-support of $\TrueParMatrix$, which is union of its supports for the different tasks. In particular, denoting $\mathcal{U}_k$ for the support of the $k$-th column of $\TrueParMatrix$, and $\mathcal{U} = \bigcup_k \mathcal{U}_k$.\\

{\bf Error bounds:}
We are also interested in providing bounds on the elementwise $\ell_\infty$ norm error of the estimator $\widehat{\ParMatrix}$,
\small
\begin{eqnarray*}
\|\widehat{\ParMatrix} - \TrueParMatrix\|_\infty = \max_{j = 1,\hdots,p}\max_{k = 1,\hdots,r} \left|\widehat{\ParMatrix}^{(k)}_j - \TrueParMatrix^{(k)}_j\right|.
\end{eqnarray*}
\normalsize

\subsection{Our Method}
Our method models the unknown parameter $\Theta$ as a superposition of a block-sparse matrix $B$ (corresponding to the features shared across many tasks) and a sparse matrix $S$ (corresponding to the features shared across few tasks). We estimate the sum of two parameter matrices $B$ and $S$ with different regularizations for each: encouraging block-structured row-sparsity in $B$ and  elementwise sparsity in $S$. The corresponding simple models would either just use block-sparse regularizations~\cite{NWJoint,Obozinski10} or just elementwise sparsity regularizations~\cite{Lasso,WainwrightLasso}, so that either method would perform better in certain suited regimes. Interestingly, as we will see in the main results, by explicitly allowing to have both block-sparse and elementwise sparse component (see Algorithm~\ref{AlgDirtyModel}), we are able to \emph{outperform both} classes of these ``clean models'', for {\em all} regimes $\TrueParMatrix$.

\begin{algorithm*}[t]\label{AlgDirtyModel}
\small
Solve the following convex optimization problem:
\begin{eqnarray}\label{EqnDirtyEstNoisy}
(\widehat{S},\widehat{B}) \in \arg\min_{S,B} && \frac{1}{2n} \sum_{k=1}^{r} 
\left\|y^{(k)} - X^{(k)} \left(s^{(k)} + b^{(k)}\right)\right\|_{2}^{2} + \lambda_s \|S\|_{1,1} + \lambda_b \|B\|_{1,\infty}.
\end{eqnarray}
Then output $\widehat{\ParMatrix} = \widehat{B} + \widehat{S}$.\\ 
\caption{Complex Block Sparse}
\end{algorithm*}

\section{Main Results and Their Consequences} \label{sec:main_results}

We now provide precise statements of our main results. A number of recent results have shown that the Lasso~\citep{Lasso,WainwrightLasso} and $\ell_1/\ell_\infty$ block-regularization~\citep{NWJoint} methods succeed in model selection, i.e., recovering signed supports with controlled error bounds under high-dimensional
scaling regimes. Our first two theorems extend these results to our model setting. In Theorem \ref{noisysupportrecoverytheorem}, we consider the case of deterministic design matrices $X^{(k)}$, and provide sufficient conditions guaranteeing signed support recovery, and elementwise $\ell_\infty$ norm error bounds. In Theorem \ref{noisysupportrecoverygaussiantheorem}, we specialize this theorem to the case where the rows of the design matrices are random from a general zero mean Gaussian distribution: this allows us to provide scaling on the number of observations required in order to guarantee signed support recovery and bounded elementwise $\ell_\infty$ norm error.

Our third result is the most interesting in that it explicitly quantifies the performance gains of our method vis-a-vis Lasso and the $\ell_1/\ell_\infty$ block-regularization method. Since this entailed finding the precise constants underlying earlier theorems, and a correspondingly more delicate analysis, we follow \citet{NWJoint} and focus on the case where there are two-tasks (i.e. $r=2$), and where we have standard Gaussian design matrices as in Theorem~\ref{noisysupportrecoverygaussiantheorem}. Further, while each of two tasks depends on $s$ features, only a fraction $\alpha$ of these are common. It is then interesting to see how the behaviors of the different regularization methods vary with the extent of overlap $\alpha$. 

\myoparagraph{Comparisons} \citet{NWJoint} show that there is actually a ``phase transition'' in the scaling of the probability of successful signed support-recovery with the number of observations. Denote a particular rescaling of the sample-size $\theta_{Lasso}(n,p,\alpha) = \frac{n}{s\log\left(p-s\right)}$. Then as \citet{WainwrightLasso} show, when the rescaled number of samples scales as $\theta_{Lasso} > 2 + \delta$ for any $\delta > 0$, Lasso succeeds in recovering the signed support of all columns with probability converging to one. But when the sample size scales as $\theta_{Lasso} < 2 - \delta$ for any $\delta > 0$, Lasso \emph{fails} with probability converging to one. For the $\ell_1/\ell_\infty$-regularized multiple linear regression, define a similar rescaled sample size  $\theta_{1,\infty}(n,p,\alpha) = \frac{n}{s\log\left(p-(2-\alpha)s\right)}$. Then as \citet{NWJoint} show there is again a transition in probability of success from near zero to near one, at the rescaled sample size of $\theta_{1,\infty} =  (4 - 3 \alpha)$. Thus, for $\alpha < 2/3$ (``less sharing'') Lasso would perform better since its transition is at a smaller sample size, while for $\alpha > 2/3$ (``more sharing'') the $\ell_1/\ell_\infty$  regularized method would perform better. 

As we show in our third theorem, the phase transition for our method occurs at the rescaled sample size of $\theta_{1,\infty} = (2 - \alpha)$, which is \emph{strictly} before either the Lasso or the $\ell_1/\ell_\infty$ regularized method except for the boundary cases: $\alpha = 0$, i.e. the case of no sharing, where we \emph{match} Lasso, and for $\alpha =1$, i.e. full sharing, where we \emph{match} $\ell_1/\ell_\infty$. Everywhere else, we {\it strictly outperform both} methods.
Figure \ref{fig1} shows the empirical performance of each of the three methods; as can be seen, they agree very well with the theoretical analysis. (Further details in the experiments Section~\ref{SecExperiments}).

\begin{figure*}[t]
\centering
\subfigure[$\alpha=0.3$]{
\includegraphics[width=0.45\linewidth]{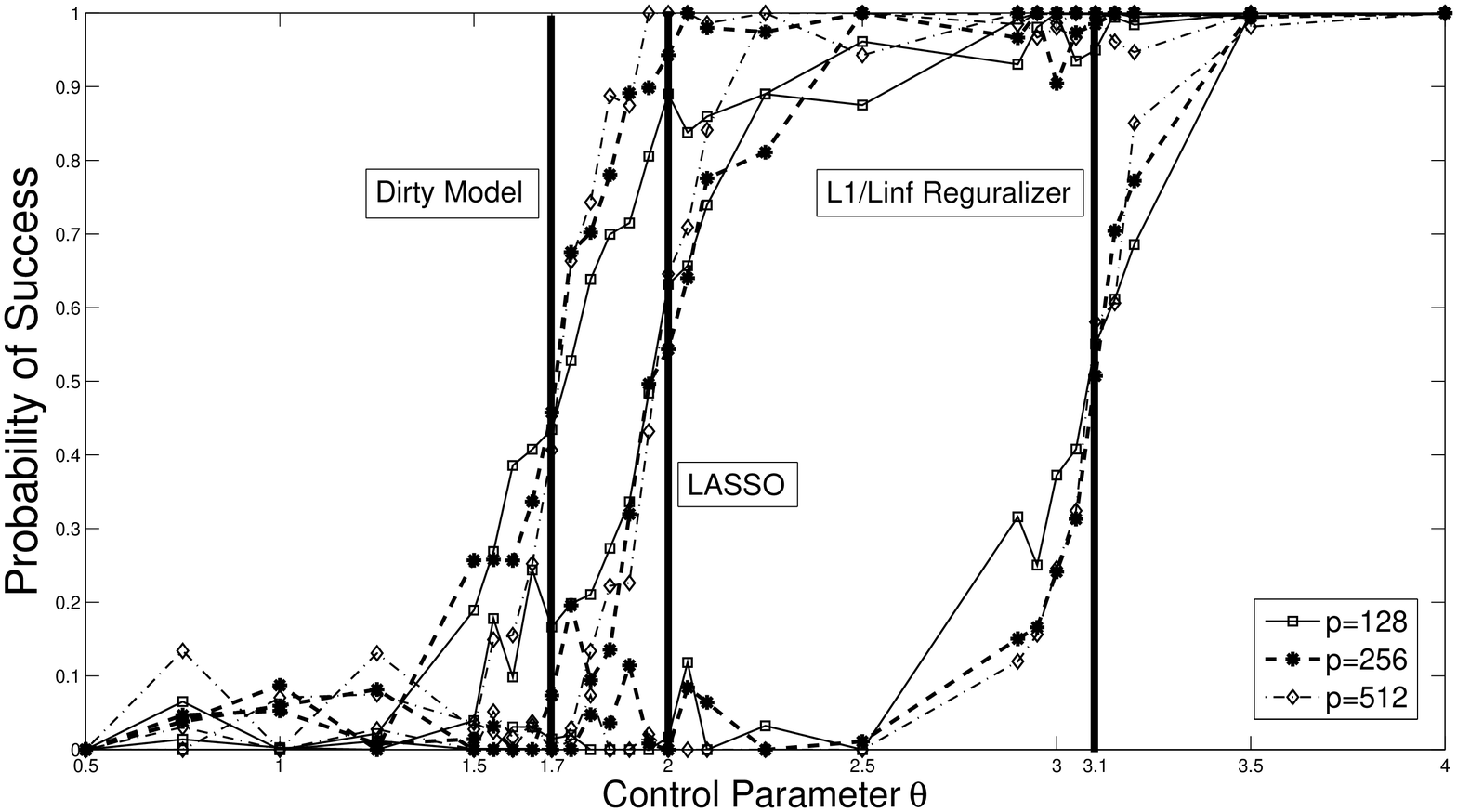}
\label{fig1:alpha03}
}
\subfigure[$\alpha=\frac{2}{3}$]{
\includegraphics[width=0.45\linewidth]{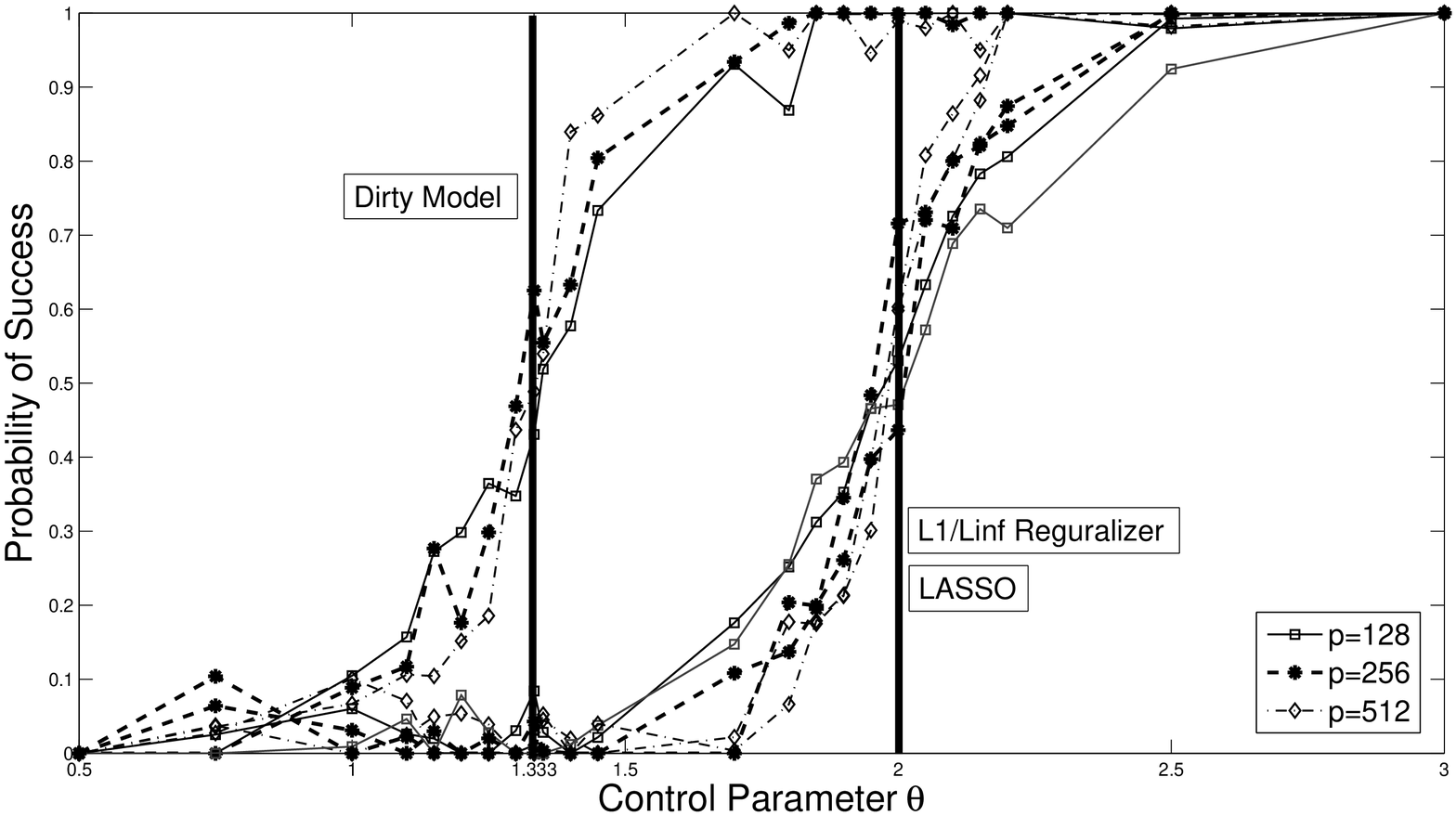}
\label{fig1:alpha066}
}
\subfigure[$\alpha=0.8$]{
\includegraphics[width=0.5\linewidth]{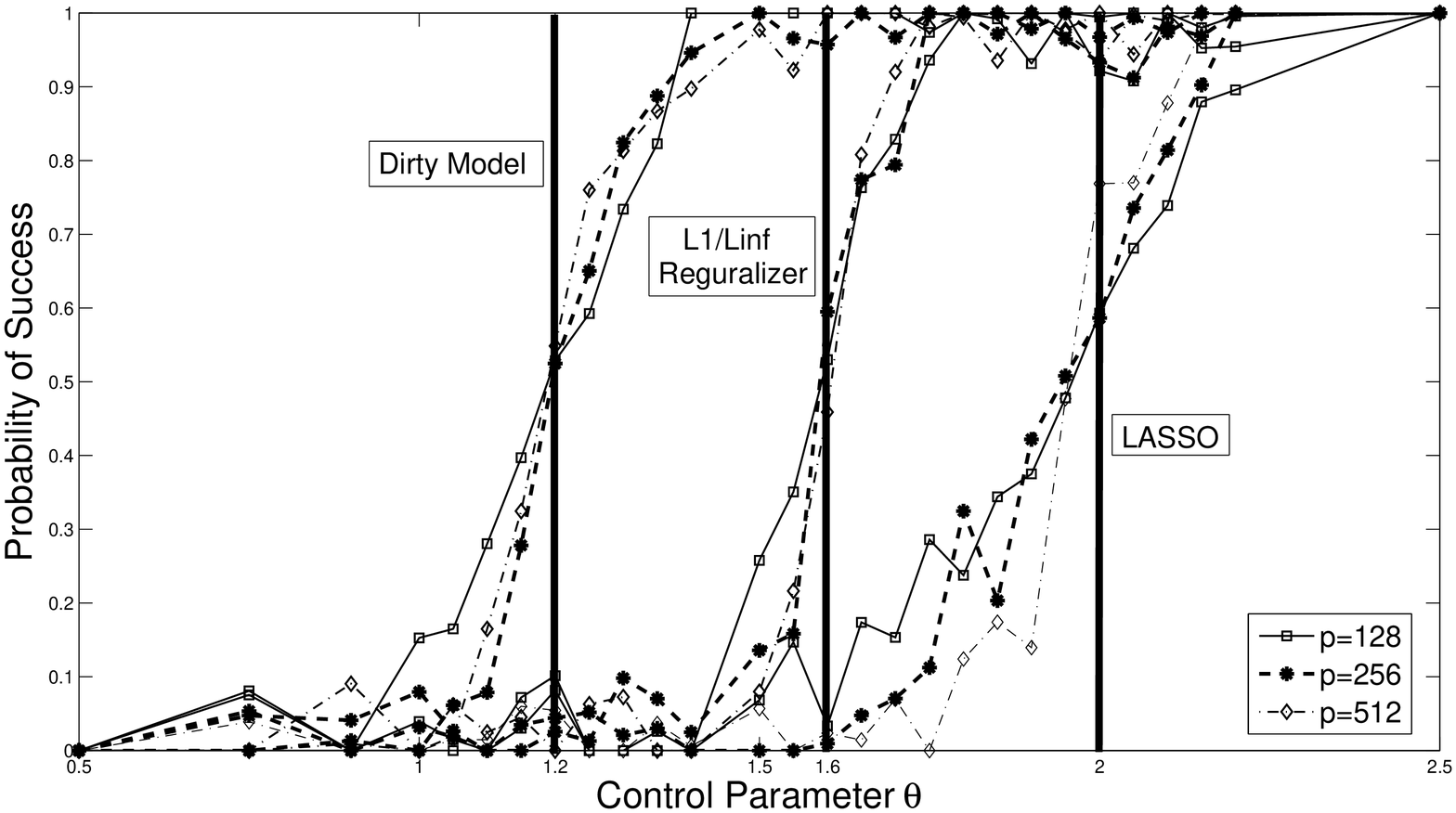}
\label{fig1:alpha08}
}
\label{fig1}
\caption{\small Probability of success in recovering the true signed support using dirty model, Lasso and $\ell_1/\ell_\infty$ regularizer. For a $2$-task problem, the probability of success for different values of feature-overlap fraction $\alpha$ is plotted. As we can see in the regimes that Lasso is better than, as good as and worse than $\ell_1/\ell_\infty$ regularizer (\subref{fig1:alpha03},  \subref{fig1:alpha066} and \subref{fig1:alpha08} respectively), the dirty model outperforms both of the methods, i.e., it requires less number of observations for successful recovery of the true signed support compared to Lasso and $\ell_1/\ell_\infty$ regularizer. Here $s=\lfloor\frac{p}{10}\rfloor$ always.}
\end{figure*}

\subsection{Sufficient Conditions for Deterministic Designs}

We first consider the case where the design matrices $X^{(k)}$ for $k=1,\cdot\cdot\cdot,r$ are deterministic, and start by specifying the assumptions we impose on the model. We note that similar sufficient conditions for the deterministic $X^{(k)}$'s case were imposed in papers analyzing Lasso~\citep{WainwrightLasso} and block-regularization methods~\cite{NWJoint,Obozinski10}.\\

\noindent {\em {\bf A0} Column Normalization:} $\|X^{(k)}_j\|_2\leq\sqrt{2n}$ for all $j = 1,\ldots,p$ and $k = 1,\ldots,r$.\\ \\

\noindent {\em {\bf A1} Incoherence Condition:}
\small$$\displaystyle\gamma_b:=1-\max_{j\in\mathcal{U}^c}\sum_{k=1}^r \left\|\tr{X_{j}^{(k)}}{X^{(k)}_{\mathcal{U}_k}\left(\tr{X^{(k)}_{\mathcal{U}_k}}{X^{(k)}_{\mathcal{U}_k}}\right)^{-1}}\right\|_1>0,$$\normalsize
where, $\mathcal{U}_k$ denotes the support of the $k$-th column of $\TrueParMatrix$, and $\mathcal{U} = \bigcup_k \mathcal{U}_k$ denotes the union of the supports of all tasks. We will also find it useful to define 
\small$$\gamma_s := 1 - \max_{1 \leq k \leq r} \max_{j \in \mathcal{U}_k^c} \left\| \tr{X^{(k)}_{j}}{X^{(k)}_{\mathcal{U}_k}} \left( \tr{X^{(k)}_{\mathcal{U}_k}}{X^{(k)}_{\mathcal{U}_k}} \right)^{-1} \right\|_1.$$\normalsize Note that by the incoherence condition {\bf A1}, we have $\gamma_s>0$. \\

\noindent {\em {\bf A2} Minimum Curvature Condition:} \small$$\displaystyle C_{min}:=\min_{1\leq k\leq r}\lambda_{min}\left(\frac{1}{n}\tr{X^{(k)}_{\mathcal{U}_k}}{X^{(k)}_{\mathcal{U}_k}}\right)>0.$$\normalsize
Also, define \small$\displaystyle D_{max}:=\max_{1\leq k\leq r}\left\|\left(\frac{1}{n}\tr{X^{(k)}_{\mathcal{U}_k}}{X^{(k)}_{\mathcal{U}_k}}\right)^{-1}\right\|_{\infty,1}$\normalsize. As a consequence of {\bf A2}, we have that $D_{\max}$ is finite.\\

\noindent {\em {\bf A3} Regularizers:} We require the regularization parameters satisfy

{\em {\bf A3-1}} $\,\,\lambda_s > \frac{2(2-\gamma_s)\sigma\sqrt{\log(pr)}}{\gamma_s\sqrt{n}}$.

{\em {\bf A3-2}} $\,\,\lambda_b > \frac{2(2-\gamma_b)\sigma\sqrt{\log(pr)}}{\gamma_b\sqrt{n}}$.

{\em {\bf A3-3}} $\,\,1\leq\frac{\lambda_b}{\lambda_s}\leq r$ and $\frac{\lambda_b}{\lambda_s}$ is not an integer (see Lemma~\ref{signBSnoisy} and \ref{Lem:non_uniqueness} for the reason).\\

\begin{theorem}
Suppose {\bf A0-A3} hold, and that we obtain estimate $\widehat{\ParMatrix}$ from our algorithm. Then, with probability at least $1 - c_1\exp(-c_2n)$, we are guaranteed that the convex program (\ref{EqnDirtyEstNoisy}) has a unique optimum and
\begin{itemize}
\item[(a)] The estimate $\widehat{\ParMatrix}$ has no false inclusions, and has bounded $\ell_\infty$ norm error:
\small\begin{equation}
\begin{aligned}
&\support(\widehat{\ParMatrix})\subseteq\support(\TrueParMatrix), \quad  \text{and} \quad\\ 
&\| \widehat{\ParMatrix} - \TrueParMatrix\|_{\infty,\infty}  \leq \underbrace{\sqrt{\frac{4\sigma^2\log\left(pr\right)}{n \, C_{min}}} + \lambda_s D_{max}}_{b_{\min}}.\\
\end{aligned}
\end{equation}\normalsize

\item[(b)]  The estimate $\widehat{\ParMatrix}$ has no false exclusions, i.e., $\sgn(\support(\widehat{\ParMatrix})) = \sgn\left(\support(\TrueParMatrix)\right)$ provided that
\footnotesize$\displaystyle\min_{(j,k) \in \support(\TrueParMatrix)} \left| \bar{\theta}_j^{(k)} \right| >  b_{\min}\,\,\,$\normalsize for $b_{\min}$ defined in part (a).
\end{itemize}
The positive constants $c_1,c_2$ depend only on $\gamma_s,\gamma_b,\lambda_s,\lambda_b$ and $\sigma$, but are otherwise independent of $n,p,r$, the problem dimensions of interest.
\label{noisysupportrecoverytheorem}
\end{theorem}

{\bf Remark:} Condition (a) guarantees that the estimate will have no \emph{false inclusions}; i.e. all included features will be relevant. If in addition, we require that it have no \emph{false exclusions} and that recover the support exactly, we need to impose the assumption in (b) that the non-zero elements are large enough to be detectable above the noise.

\subsection{General Gaussian Designs}

Often the design matrices consist of samples from a Gaussian ensemble (e.g. in Gaussian graphical model structure learning). Suppose that for each task $k = 1,\ldots,r$ the design matrix $X^{(k)}\in\mathbb{R}^{n\times p}$ is such that each row $X^{(k)}_i\in\mathbb{R}^{p}$ is a zero-mean Gaussian random vector with  covariance matrix $\Sigma^{(k)}\in\mathbb{R}^{p\times p}$, and is independent of every other row. Let $\Sigma^{(k)}_{\mathcal{V},\mathcal{U}}\in\mathbb{R}^{\left|\mathcal{V}\right|\times\left|\mathcal{U}\right|}$ be the submatrix of $\Sigma^{(k)}$ with corresponding rows to $\mathcal{V}$ and columns to $\mathcal{U}$. We require these covariance matrices to satisfy the following conditions: \\

\noindent {\em {\bf C1} Incoherence Condition:} \small$$\displaystyle\gamma_b:=1-\max_{j\in\mathcal{U}^c}\sum_{k=1}^r \left\|\Sigma^{(k)}_{j,\mathcal{U}_k},\left(\Sigma^{(k)}_{\mathcal{U}_k,\mathcal{U}_k}\right)^{-1}\right\|_1>0\,$$\normalsize. \\
 
\noindent {\em {\bf C2} Minimum Curvature Condition:} \small$$\displaystyle C_{min}:=\min_{1\leq k\leq r}\lambda_{min}\left(\Sigma^{(k)}_{\mathcal{U}_k,\mathcal{U}_k}\right)>0\,$$\normalsize and let \small$\displaystyle D_{max}:=\left\|\left(\Sigma^{(k)}_{\mathcal{U}_k,\mathcal{U}_k}\right)^{-1}\right\|_{\infty,1}\,$\normalsize.\\

These conditions are analogues of the conditions for deterministic designs; they are now imposed on the covariance matrix of the (randomly generated) rows of the design matrix.\\

\noindent {\em {\bf C3} Regularizers:} Defining $s:=\max_k\left|\mathcal{U}_k\right|$, we require the regularization parameters satisfy

{\em {\bf C3-1}} $\,\,\lambda_s \geq \frac{\mysqrt{4\sigma^2C_{min}\log(pr)}}{\gamma_s\sqrt{nC_{min}}-\sqrt{2s\log(pr)}}$.

{\em {\bf C3-2}} $\,\,\lambda_b \geq \frac{\mysqrt{4\sigma^2C_{min}r (r\log(2)+\log(p))}}{\gamma_b\sqrt{nC_{min}}-\sqrt{2sr(r\log(2)+\log(p))}}$.

{\em {\bf C3-3}} $\,\,1\leq\frac{\lambda_b}{\lambda_s}\leq r$ and $\frac{\lambda_b}{\lambda_s}$ is not an integer.\\

\begin{theorem}
Suppose assumptions {\bf C1-C3} hold, and that the number of samples scale as  
\small$$n>\max\left(\frac{2s\log(pr)}{C_{min}\gamma_s^2},\frac{2sr\big(r\log(2)+\log(p)\big)}{C_{min}\gamma_b^2}\right).$$\normalsize
Suppose we obtain estimate $\widehat{\ParMatrix}$ from our algorithm. Then, with probability at least $$1-c_1\exp\left(-c_2\left(r\log(2)+\log(p)\right)\right)-c_3\exp(-c_4\log(rs)) \rightarrow 1$$ for some positive numbers $c_1-c_4$, we are guaranteed that the algorithm estimate $\widehat{\ParMatrix}$ is unique and satisfies the following conditions:
\begin{itemize}
\item[(a)]
The estimate $\widehat{\ParMatrix}$ has no false inclusions, and has bounded $\ell_\infty$ norm error so that
\small\begin{equation}
\begin{aligned}
&\support(\widehat{\ParMatrix})\subseteq\support(\TrueParMatrix), \quad  \text{and} \\\ 
&\| \widehat{\ParMatrix} - \TrueParMatrix\|_{\infty,\infty}  \leq \underbrace{\sqrt{\frac{50\sigma^2\log(rs)}{nC_{min}}}+\lambda_s\left(\frac{4s}{C_{min}\sqrt{n}}+D_{max}\right)}_{g_{\min}}.\\
\end{aligned}
\end{equation}\normalsize

\item[(b)]  The estimate $\widehat{\ParMatrix}$ has no false exclusions, i.e., $\sgn(\support(\widehat{\ParMatrix}) ) = \sgn\left(\support(\TrueParMatrix)\right)$ provided that
\footnotesize$\displaystyle\min_{(j,k) \in \support(\TrueParMatrix)} \left| \bar{\theta}_j^{(k)} \right| > g_{\min}\,$\normalsize for $g_{\min}$ defined in part (a).
\end{itemize}
\label{noisysupportrecoverygaussiantheorem}
\end{theorem}

\subsection{Quantifying the gain for $2$-Task Gaussian Designs}

This is one of the most important results of this paper. Here, we perform a more delicate and finer analysis to establish precise quantitative gains of our method. We focus on the special case where $r=2$ and the design matrix has rows generated from the standard Gaussian distribution $\mathcal{N}(0,I_{n\times n})$. As we will see both analytically and experimentally, our method strictly outperforms both Lasso and $\ell_1/\ell_\infty$-block-regularization over for all cases, except at the extreme endpoints of no support sharing (where it matches that of Lasso) and full support sharing (where it matches that of $\ell_1/\ell_\infty$). We now present our analytical results; the empirical comparisons are presented next in Section \ref{SecExperiments}. The results will be in terms of a particular rescaling of the sample size $n$ as 
\begin{equation}
\theta(n,p,s,\alpha):=\frac{n}{(2-\alpha)s\log\left(p-(2-\alpha)s\right)}.
\nonumber
\end{equation}
We also require that the regularizers satisfy

\noindent {\em {\bf F1}} \footnotesize$\displaystyle \lambda_s > \frac{\mysqrt{4\sigma^2 (1-\sqrt{s/n}) (\log(r)+\log(p-(2-\alpha)s))}} 
{\sqrt{n}- \sqrt{s}- \mysqrt{ (2 - \alpha) \; s\; (\log(r)+\log(p-(2-\alpha)s))}}\,$\normalsize.\\

\noindent {\em {\bf F2}} \footnotesize$\displaystyle \lambda_b > \frac{\mysqrt{4\sigma^2 (1 - \sqrt{s/n}) r (r \log(2) + \log (p-(2-\alpha)s))}} 
{\sqrt{n} - \sqrt{s} - \mysqrt{ (1-\alpha/2 )\; sr \; (r\log(2) + \log(p-(2-\alpha)s))}}\,$\normalsize. \\

\noindent {\em {\bf F3}} $\frac{\lambda_b}{\lambda_s}=\sqrt{2}$.\\

\begin{theorem}
Consider a $2$-task regression problem $(n,p,s,\alpha)$, where the design matrix has rows generated from the standard Gaussian distribution $\mathcal{N}(0,I_{n\times n})$.  Suppose \footnotesize $$\max_{j\in B^*}\Bigg|\left|\Theta^{*(1)}_j\right|-\left|\Theta^{*(2)}_j\right|\Bigg| \leq c\lambda_s,$$\normalsize where, $B^{*}$ is the submatrix of $\Theta^*$ with rows where both entries are non-zero and $c$ is a constant specified in Lemma~\ref{LemGapLambda}. Then the
estimate $\widehat{\ParMatrix}$ of the problem~\eqref{EqnDirtyEstNoisy} satisfies the following: 
\begin{itemize}
\item [(\textbf{Success})] Suppose the regularization coefficients satisfy ${\bf F1-F3}$. Further, assume that the number of samples scales as $\theta(n,p,s,\alpha) > 1$. Then, with probability at least $1-c_1\exp(-c_2n)$ for some positive numbers $c_1$ and $c_2$, we are guaranteed that $\widehat{\ParMatrix}$ satisfies the support-recovery and $\ell_\infty$ error bound conditions (a-b) in Theorem~\ref{noisysupportrecoverygaussiantheorem}.
\item [(\textbf{Failure})] If $\theta(n,p,s,\alpha) < 1$ there is no solution $(\hat{B},\hat{S})$ for any choices of $\lambda_s$ and $\lambda_b$ such that $\sgn\left(\support(\widehat{\ParMatrix})\right)=\sgn\left(\support(\TrueParMatrix)\right)$.
\end{itemize}
\label{twotasktheorem}
\end{theorem}
{\bf Remark:} The assumption on the gap \footnotesize  $\Big|\left|\Theta^{*(1)}_j\right|-\left|\Theta^{*(2)}_j\right|\Big|\leq c\lambda_s$ \normalsize reflects the fact that we require that most values of $\Theta^*$ to be balanced on both tasks on the shared support. As we show in a more general theorem (Theorem~4) in Section~\ref{appendix:proof_thm3}, even in the case where the gap is large, the dependence of the sample scaling on the gap is quite weak.

\section{Simulation Results}\label{SecExperiments}
In this section, we provide some simulation results. First, using our synthetic data set, we investigate the consequences of Theorem 3 when we have $r=2$ tasks to learn. As we see, the empirical result verifies our theoretical guarantees. Next, we apply our method regression to a real datasets: a hand-written digit classification dataset with $r=10$ tasks (equal to the number of digits $0-9$). For this dataset, we show that our method outperforms both LASSO and $\ell_1/\ell_\infty$ practically. For each method, the parameters are chosen via cross-validation; see supplemental material for more details.

\subsection{Synthetic Data Simulation}
Consider a $r=2$-task regression problem of the form $(n,p,s,\alpha)$ as discussed in Theorem 3. For a fixed set of parameters $(n,s,p,\alpha)$, we generate $100$ instances of the problem. Then, we solve the same problem using our model, $\ell_1/\ell_\infty$ regularizer and LASSO by searching for penalty regularizer coefficients independently for each one of these programs to find the best regularizer by cross validation. After solving the three problems, we compare the signed support of the solution with the true signed support and decide whether or not the program was successful in signed support recovery. We describe these process in more details in this section.\\

\textbf{Data Generation}: We explain how we generated the data for our simulation here. We pick three different values of $p=128,256,512$ and let $s=\lfloor 0.1p\rfloor$. For different values of $\alpha$, we let $n=c\,s\log(p-(2-\alpha)s)$ for different values of $c$. We generate a random \underline{sign} matrix $\widetilde{\Theta}^*\in\mathbb{R}^{p\times 2}$ (each entry is either $0$, $1$ or $-1$) with column support size $s$ and row support size $(2-\alpha)s$ as required by Theorem 3. Then, we multiply each row by a real random number with magnitude greater than the minimum required for sign support recovery by Theorem 3. We generate two sets of matrices $X^{(1)}$, $X^{(2)}$ and $W$ and use one of them for training and the other one for cross validation (test), subscripted $\text{Tr}$ and $\text{Ts}$, respectively. Each entry of the noise matrices $W_{\text{Tr}},W_{\text{Ts}}\in\mathbb{R}^{n\times 2}$ is drawn independently according to $\mathcal{N}(0,\sigma^2)$ where $\sigma=0.1$. Each row of a design matrix $X^{(k)}_{\text{Tr}},X^{(k)}_{\text{Ts}}\in\mathbb{R}^{n\times p}$ is sampled, independent of any other rows, from $\mathcal{N}(0,\mathbf{I}_{2\times 2})$ for all $k=1,2$. Having $X^{(k)}$, $\bar{Theta}$ and $W$ in hand, we can calculate $Y_{\text{Tr}},Y_{\text{Ts}}\in\mathbb{R}^{n\times 2}$ using the model $y^{(k)}=X^{(k)}\theta^{(k)}+w^{(k)}$ for all $k=1,2$ for both train and test set of variables.\\

\textbf{Coordinate Descent Algorithm}: Given the generated data $X^{(k)}_{\text{Tr}}$ for $k=1,2$ and $Y_{\text{Tr}}$ in the previous section, we want to recover matrices $\Bh$ and $\Sh$ that satisfy \eqref{EqnDirtyEstNoisy}. We use the coordinate descent algorithm to numerically solve the problem (see Appendix~\ref{appendix:coordinate_descent}). The algorithm inputs the tuple $(X^{(1)}_{\text{Tr}}, X^{(2)}_{\text{Tr}}, Y_{\text{Tr}}, \lambda_s, \lambda_b, \epsilon, \underline{B}, \underline{S})$ and outputs a matrix pair $(\Bh,\Sh)$. The inputs $(\underline{B},\underline{S})$ are initial guess and can be set to zero. However, when we search for optimal penalty regularizer coefficients, we can use the result for previous set of coefficients $(\lambda_b,\lambda_s)$ as a good initial guess for the next coefficients $(\lambda_b+\xi,\lambda_s+\zeta)$. The parameter $\epsilon$ captures the stopping criterion threshold of the algorithm. We iterate inside the algorithm until the relative update change of the objective function is less than $\epsilon$. Since we do not run the algorithm completely (until $\epsilon=0$ works), we need to filter the small magnitude values in the solution $(\Bh,\Sh)$ and set them to be zero. \\

\textbf{Choosing penalty regularizer coefficients}: Dictated by optimality conditions, we have $1>\frac{\lambda_s}{\lambda_b}>\frac{1}{2}$. Thus, searching range for one of the coefficients is bounded and known. We set $\lambda_b=c\sqrt{\frac{rlog(p)}{n}}$ and search for $c\in[0.01,100]$, where this interval is partitioned logarithmic. For any pair $(\lambda_b,\lambda_s)$ we compute the objective function of $Y_{\text{Ts}}$ and $X^{(k)}_{\text{Ts}}$ for $k=1,2$ using the filtered $(\Bh,\Sh)$ from the coordinate descent algorithm. Then across all choices of $(\lambda_b,\lambda_s)$, we pick the one with minimum objective function on the test data. Finally we let $\hat{\Theta}=\text{Filter}(\Bh+\Sh)$ for $(\Bh,\Sh)$ corresponding to the optimal $(\lambda_b,\lambda_s)$.\\

\textbf{Performance Analysis}: We ran the algorithm for five different values of the overlap ratio $\alpha\in\{0.3, \frac{2}{3}, 0.8\}$ with three different number of features $p\in\{128, 256, 512\}$. For any instance of the problem $(n,p,s,\alpha)$, if the recovered matrix $\hat{\Theta}$ has the same sign support as the true $\bar{\Theta}$, then we count it as success, otherwise failure (even if one element has different sign, we count it as failure). 

As Theorem 3 predicts and Fig~\ref{fig1} shows, the right scaling for the number of oservations is $\frac{n}{s\log(p-(2-\alpha)s)}$, where all curves stack on the top of each other at $2-\alpha$. Also, the number of observations required by our model for true signed support recovery is always less than both LASSO and $\ell_1/\ell_\infty$ regularizer. Fig~\ref{fig1:alpha03} shows the probability of success for the case $\alpha=0.3$ (when LASSO is better than $\ell_1/\ell_\infty$ regularizer) and that our model outperforms both methods. When $\alpha=\frac{2}{3}$ (see Fig~\ref{fig1:alpha066}), LASSO and $\ell_1/\ell_\infty$ regularizer performs the same; but our model require almost $33\%$ less observations for the same performance. As $\alpha$ grows toward $1$, e.g. $\alpha=0.8$ as shown in Fig~\ref{fig1:alpha08}, $\ell_1/\ell_\infty$ performs better than LASSO. Still, our model performs better than both methods in this case as well.\\

\textbf{Scaling Verification}: To verify that the phase transition threshold changes linearly with $\alpha$ as predicted by Theorem 3, we plot the phase transition threshold versus $\alpha$. For five different values of $\alpha\in\{0.05,0.3,\frac{2}{3},0.8,0.95\}$ and three different values of $p\in\{128,256,512\}$, we find the phase transition threshold for our model, LASSO and $\ell_1/\ell_\infty$ regularizer. We consider the point where the probability of success in recovery of signed support exceeds $50\%$ as the phase transition threshold. We find this point by interpolation on the closest two points. Fig~\ref{fig2} shows that phase transition threshold for our model is always lower than the phase transition for LASSO and $\ell_1/\ell_\infty$ regularizer.\\

\begin{figure}[t]
\centering
\includegraphics[width=3.5in]{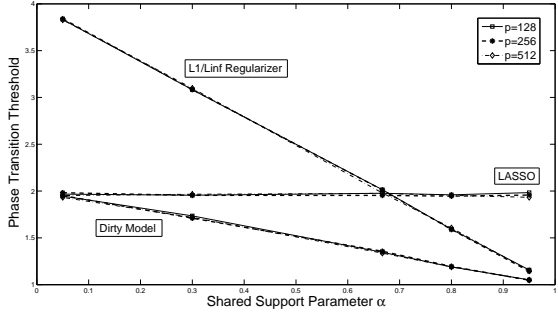}
\caption{\small Verification of the result of the Theorem 3 on the behavior of phase transition threshold by changing the parameter $\alpha$ in a $2$-task $(n,p,s,\alpha)$ problem for our method, LASSO and $\ell_1/\ell_\infty$ regularizer. The $y$-axis is $\frac{n}{s\log(p-(2-\alpha)s)}$, where $n$ is the number of samples at which threshold was observed. Here $s=\lfloor\frac{p}{10}\rfloor$. Our method shows a gain in sample complexity over the entire range of sharing $\alpha$. The pre-constant in Theorem \ref{twotasktheorem} is also validated.}
\label{fig2}
\end{figure}

\subsection{Handwritten Digits Dataset}
We use a handwritten digit dataset to illustrate the performance of our method. According to the description of the dataset, this dataset consists of features of handwritten numerals ($0$-$9$) extracted from a collection of Dutch utility maps \cite{DSet}. This dataset has been used by a number of papers \cite{BREDUITAXHAR,HENIY} as a reliable dataset for handwritten recognition algorithms.\\

\textbf{Structure of the Dataset}: In this dataset, there are 200 instances of handwritten digits $0$-$9$ (totally 2000 digits). Each instance of each digit is scanned to an image of the size $30\times 48$ pixels. This image is NOT provided by the dataset. Using the full resolution image of each digit, the dataset provides six different classes of features. A total of $649$ features are provided for each instance of each digit. The information about each class of features is provided in Table~\ref{ta1b}. The combined handwriting images of the record number $100$ is shown in Fig~\ref{exp1fig1} (ten images are concatenated together with a spacer between each two).\\

\begin{table*}
\centering
\begin{tabular}{|c||c|c|c|c|}
\hline
& \textbf{Feature} & \textbf{Size} & \textbf{Type} & \textbf{Dynamic Range}\\
\hline\hline
1 & Pixel Shape ($15\times 16$) & 240 & Integer & 0-6\\
\hline
2 & 2D Fourier Transform Coefficients & 74 & Real & 0-1\\
\hline
3 & Karhunen-Loeve Transform Coeficients & 64 & Real & -17:17\\
\hline
4 & Profile Correlation & 216 & Integer & 0-1400\\
\hline
5 & Zernike Moments & 46 & Real & 0-800\\
\hline 
&& 3 & Integer & 0-6 \\
\cline{3-5}
6 & Morphological Features & 1 & Real& 100-200\\
\cline{3-5}
&& 1 & Real & 1-3 \\
\cline{3-5}
&& 1 & Real & 1500-18000\\
\hline
\end{tabular}
\caption{Six different classes of features provided in the dataset. The dynamic ranges are approximate not exact. The dynamic range of different morphological features are completely different. For those $6$ morphological features, we provide their different dynamic ranges separately.}
\label{ta1b}
\end{table*}

\begin{figure}[t]
\centering
\includegraphics[width=3in]{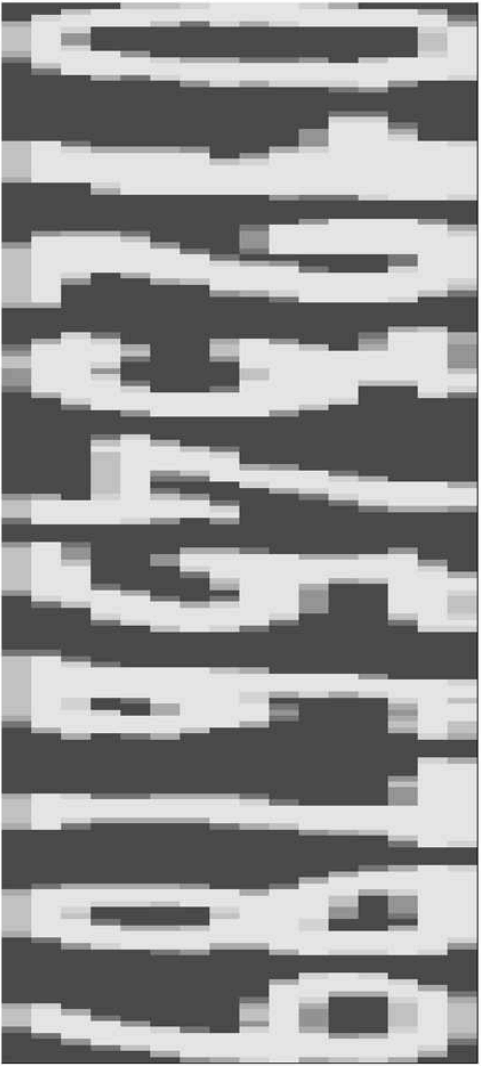}
\caption{An instance of images of the ten digits extracted from the dataset}
\label{exp1fig1}
\end{figure}

\textbf{Fitting the dataset to our model}: Regardless of the nature of the features, we have $649$ features for each of $200$ instance of each digit. We need to learn $K=10$ different tasks corresponding to ten different digits. To make the associated numbers of features comparable, we shrink the dynamic range of each feature to the interval $-1$ and $1$. We divide each feature by an appropriate number (perhaps larger than the maximum of that feature in the dataset) to make sure that the dynamic range of all features is a (not too small) subset of $[-1,1]$. Notice that in this division process, we don't care about the minimum and maximum of the training set. We just divide each feature by a fixed and predetermined number we provided as maximum in Table~\ref{ta1b}. For example, we divide the Pixel Shape feature by $6$, Karhunen-Loeve coefficients by $17$ or the last morphological feature by $18000$ and so on. We do not shift the data; we only scale it.

Out of $200$ samples provided for each digit, we take $n\leq 200$ samples for training. Let $X^{(k)}=X\in\mathbb{R}^{10n\times 649}$ for all $0\leq k\leq 9$ be the matrix whose first $n$ rows correspond to the features of the digit $0$, the second $n$ rows correspond to the features of the digit $1$ and so on. Consequently, we set the vector $y^{(k)}\in\{0,1\}^{10n}$ to be the vector such that $y^{(k)}_j=1$ if and only if the $j^{th}$ row of the feature matrix $X$ corresponds to the digit $k$. This setup is called binary classification setup.

We want to find a block-sparse matrix $\Bh\in
\mathbb{R}^{649\times 10}$ 
and a sparse matrix $\Sh\in\mathbb{R}^{649\times 10}$, so that for a given feature vector 
$\mathbf{x}\in\mathbb{R}^{649}$ extracted from the image of 
a handwritten digit $0\leq k^*\leq 9$, we 
ideally have $k^*=\arg\max_{0\leq k\leq 9}\mathbf{x}\left(\Bh+\Sh\right)$.

To find such matrices $\Bh$ and $\Sh$, we solve \eqref{EqnDirtyEstNoisy}. We tune the parameters $\lambda_b$ and $\lambda_s$ in order to get the best result by cross validation. Since we have $10$ tasks, we search for $\frac{\lambda_s}{\lambda_b}\in\left[\frac{1}{10},1\right]$ and let $\lambda_b=c\sqrt{\frac{2log(649)}{n}}\approx\frac{5c}{\sqrt{n}}$, where, empirically $c\in[0.01,10]$ is a constant to be searched.\\

\textbf{Performance Analysis}: Table~\ref{tab2} shows the results of our analysis for different sizes of the training set as $\frac{n}{200}$. We measure the classification error on the test set for each digit to get the $10$-vector of errors. Then, we find the average error and the variance of the error vector to show how the error is distributed over all tasks. We compare our method with $\ell_1/\ell_{\infty}$ reguralizer method and LASSO.\\

\begin{table*}
\centering
\scriptsize
\begin{tabular}{|c||c|c|c|c|}
\hline
\textbf{$\frac{n}{200}$}&&\textbf{Our Model} &\textbf{$\ell_1/\ell_\infty$}& \textbf{LASSO}\\
\hline\hline
5\% & Average Classification Error & 8.6\% & 9.9\% & 10.8\%\\ 
& Variance of Error & 0.53\% & 0.64\% & 0.51\%\\
& Average Row Support Size & $B$:165  $\qquad B+S$:171 & 170 & 123\\
& Average Support Size & $S$:18  $\qquad B+S$:1651 & 1700 & 539\\
\hline
10\% & Average Classification Error & 3.0\% & 3.5\% & 4.1\%\\ 
& Variance of Error & 0.56\% & 0.62\% & 0.68\%\\
& Average Row Support Size & $B$:211 $\qquad B+S$:226 & 217 & 173\\
& Average Support Size & $S$:34 $\qquad B+S$:2118 & 2165 & 821\\
\hline
20\% & Average Classification Error & 2.2\% & 3.2\% & 2.8\%\\ 
& Variance of Error & 0.57\% & 0.68\% & 0.85\%\\
& Average Row Support Size & $B$:270  $\qquad B+S$:299 & 368 & 354\\
& Average Support Size & $S$:67  $\qquad B+S$:2761 & 3669 & 2053\\
\hline
\end{tabular}
\caption{\small Simulation Results for our model, $\ell_1/\ell_\infty$ and LASSO.}
\label{tab2}
\end{table*}

\section{Proof Outline}
\noindent In this section we illustrate the proof outline of all three theorems as they are very similar in the nature. First, we introduce some notations and definitions and then, we provide a three step proof technique that we used to prove all three theorems.

\subsection{Definitions and Setup}
\label{appendix:definitions}
In this section, we rigorously define the terms and notation we used throughout the proofs. \\

\noindent \textbf{Notation}: For a vector $v$, the norms $\ell_1$, $\ell_2$ and $\ell_\infty$ are denoted as $\|v\|_1=\sum_{k}\left|v^{(k)}\right|$, $\|v\|_2=\sqrt{\sum_{k}\left|v^{(k)}\right|^2}$ and $\|v\|_\infty=\max_{k}\left|v^{(k)}\right|$, respectively. Also, for a matrix $Q\in\mathbb{R}^{p\times r}$, the norm $\ell_\zeta/\ell_\rho$ is denoted as $\|Q\|_{\rho,\zeta}=\|\left(\|q_1\|_{\zeta},\cdot\cdot\cdot,\|q_p\|_{\zeta}\right)\|_{\rho}$. The maximum singular value of $Q$ is denoted as $\lambda_{max}(Q)$. For a matrix $X\in\mathbb{R}^{n\times p}$ and a set of indices $\mathcal{U}\subseteq\{1,\cdot\cdot\cdot,p\}$, the matrix $X_{\mathcal{U}}\in\mathbb{R}^{n\times\left|\mathcal{U}\right|}$ represents the sub-matrix of $X$ consisting of $X_j$'s where $j\in\mathcal{U}$.\\

\subsubsection{Towards Identifying Optimal Solution}
\label{appendix:H_Transform}
This is a key step in our analysis. Our proof proceeds by choosing a pair $\widehat{B},\widehat{S}$ such that the signed support of $\widehat{B} +\widehat{S}$ is the same as that of $\bar{\Theta}$, and then certifying that, under our assumptions, this pair is the optimum of  the optimization problem~\eqref{EqnDirtyEstNoisy}. We construct this pair via a surrogate optimization problem -- dubbed {\em oracle problem} in the literature as well as our proof outline below -- which adds extra constraints to \eqref{EqnDirtyEstNoisy} in a way that ensures signed support recovery. Making the oracle problem is a key step in our proof. 

For \eqref{EqnDirtyEstNoisy}, let $d=\lceil\frac{\lambda_b}{\lambda_s}\rceil$; in this paper we will always have $1\leq d \leq r$, where we recall $r$ is the number of tasks. Using this $d$, we now define two matrices $\Bst,\Ss$, such that $\Bst + \Ss = \bar{\Theta}$, as follows. In each row $\bar{\Theta}_j$, let $v_j$ be the $(d+1)^{th}$ largest magnitude of the elements in $\Theta_j$. Then, the $(j,k)^{th}$ element $s^{*(k)}_j$ of the matrix $\Ss$ is defined as follows
\[
s^{*(k)}_j ~ = ~ \sgn(\theta^{(k)}_j) \max \left \{ 0, \left | \theta^{(k)}_j \right | - v_j \right \}
\]
In words, to obtain $\Ss$ we take the matrix $\bar{\Theta}$ and for each element  we {\em clip its magnitude} to be the {\em excess} over the $(d+1)^{th}$ largest magnitude in its row. We retain the sign.  Finally, define $\Bst = \bar{\Theta} - \Ss$ to be the residual. It is thus  clear that 
\begin{itemize}
\item $\Ss$ will have at most $d$ non-zero elements in each row.
\item Each row of $\Bst$ is either identically 0, or has at least $d$ non-zero elements. Also, in the latter case, at least $d$ of them have the same magnitude.
\item If any element $(j,k)$ is non-zero in both $\Ss$ and $\Bst$ then its sign is the same in both.
\end{itemize}
$\Ss$ thus takes on the role of the ``true sparse matrix", and $\Bst$ the role of the ``true block-sparse matrix". We will use $\Bst,\Ss$ to construct our oracle problem later. The pair also has the following significance: our results will imply that if we have infinite samples, then $\Bst,\Ss$ will be the solution to  \eqref{EqnDirtyEstNoisy}. 
 
\subsubsection{Sparse Matrix Setup}
For any matrix $S$, define $\support(S)=\{(j,k):s_{j}^{(k)}\neq 0\}$, and let $U_s=\{S\in\mathbb{R}^{p\times r}:\support(S)\subseteq\support(\Ss)\}$ be the subspace of matrices whose their support is the subset of the matrix $\Ss$. The orthogonal projection to the subspace $U_s$ can be defined as follows:
\begin{equation}
\left(P_{U_s}(S)\right)_{j,k}=\left\{
\begin{aligned}
&s_{j}^{(k)}&(j,k)\in\support(\Ss)\\
&0&\text{ow}.
\end{aligned}\right.
\nonumber
\end{equation}
We can define the orthogonal complement space of $U_s$ to be $U_s^c=\{S\in\mathbb{R}^{p\times r}:\support(S)\cap\support(\Ss)=\phi\}$. The orthogonal projection to this space can be defined as $P_{U_s^c}(S)=S-P_{U_s}(S)$. Since the type of the block-sparsity we consider is a block-sparsity assumption on the rows of matrices, we need to characterize the sparsity of the rows of the matrix $\Ss$. This motivates to define $D(S) = \max_{1\leq j\leq p} \|s_{j}\|_{0}$ denoting the maximum number of non-zero elements in any row of the sparse matrix $S$.\\

\subsubsection{Row-Sparse Matrix Setup}
\noindent For any matrix $B$, define $\rowsupport(B)=\{j:\exists k \quad\text{s.t.}\quad b_{j}^{(k)}\neq 0\}$, and let $U_b=\{B\in\mathbb{R}^{p\times r}:\rowsupport(B)\subseteq\rowsupport(\Bst)\}$ be the subspace of matrices whose their row support is the subset of the row support of the matrix $\Bst$. The orthogonal projection to the subspace $U_b$ can be defined as follows:
\begin{equation}
\left(P_{U_b}(B)\right)_j=\left\{
\begin{aligned}
&b_{j}&\qquad j\in\rowsupport(\Bst)\\
&\mathbf{0}&\qquad\text{ow}.
\end{aligned}\right.
\nonumber
\end{equation}
We can define the orthogonal complement space of $U_b$ to be $U_b^c=\{B\in\mathbb{R}^{p\times r}:\rowsupport(B)\cap\rowsupport(\Bst)=\phi\}$. The orthogonal projection to this space can be defined as $P_{U_b^c}(B)=B-P_{U_b}(B)$.

\noindent For a given matrix $B\in\mathbb{R}^{p\times r}$, let $M_j(B)=\{k: |b_{j}^{(k)}|=\|b_j\|_{\infty}>0\}$ be the set of indices that the corresponding elements achieve the maximum magnitude on the $j^{th}$ row with positive or negative signs. Also, let $M(B)=\min_{1\leq j\leq p}|M_{j}(B)|$ be the minimum number of elements who achieve the maximum in each row of the matrix $B$.\\ \\

The following technical lemma is useful in the proof of all three theorems.

\begin{lemma}
If $(B,S)=\mathcal{H}_d(\Theta)$ then
\begin{itemize}
\item [(P1)] $M(B)\geq d+1$ and $D(S)\leq d$.
\item [(P2)] $\sgn(s_j^{(k)})=\sgn(b_j^{(k)})$ for all $j\in\rowsupport(B)$ and $k\in M_j(B)$.
\item [(P3)] $s_j^{(k)}=0$ for all $j\in\rowsupport(B)$ and $k\notin M_j(B)$.
\end{itemize}
\label{Lem:Star_Properties}
\end{lemma}

\begin{proof}
The proof follows from the definition of $\mathcal{H}$.\\
\end{proof}

\subsection{Proof Overview}
The proofs of all three of our theorems follow a primal-dual witness technique, and consist of two steps, as detailed in this section. The first step
constructs a primal-dual witness candidate, and is common to all three theorems. The second step consists of showing that the candidate constructed in the first step is indeed 
a primal-dual witness. The theorem proofs differ in this second step, and show that under the respective conditions imposed in the theorems, the construction succeeds with high probability. 
These steps are as follows:\\

{\bf STEP 1:} Denote the true optimal solution pair $(B^*,S^*)=\mathcal{H}_d(\bar{\Theta})$ as defined in Section~\ref{appendix:H_Transform}, for $d=\lfloor\frac{\lambda_b}{\lambda_s}\rfloor$. See Lemma~\ref{Lem:Star_Properties} for basic properties of these matrices $B^*$ and $S^*$.

{\bf Primal Candidate:}
We can then design a candidate optimal solution $(\tilde{S},\tilde{B})$ with the desired sparsity pattern using a restricted support optimization problem, called \emph{oracle problem}:

\begin{equation}\label{EqnOracleEst}
\begin{aligned}
(\tilde{S},\tilde{B}) \in \arg\min_{S\in U_s,B\in U_b} &\frac{1}{2n} \sum_{k=1}^{r} \left\|y^{(k)} - X^{(k)}\left(s^{(k)} + b^{(k)}\right)\right\|_{2}^{2}\\ &\qquad+ \lambda_s \|S\|_{1,1} + \lambda_b \|B\|_{1,\infty}.\\
\end{aligned}
\end{equation}

{\bf Dual Candidate:}
We set $\EstDual_{\bigcup_{k=1}^r\mathcal{U}_k}$ as the subgradient of the optimal primal parameters of \eqref{EqnOracleEst} . Specifically, 
we set $$ \EstDual_{\bigcup_{k=1}^r\mathcal{U}_k} = \left(\EstDual_s\right)_{\bigcup_{k=1}^r\mathcal{U}_k} + \left(\EstDual_b\right)_{\bigcup_{k=1}^r\mathcal{U}_k},$$ where, $\EstDual_s = \lambda_s \sgn(\tilde{S})$, and for all $j\in\bigcup_{k=1}^r\mathcal{U}_k$,
\begin{equation}
(\tilde{z}_b)_{j}^{(k)}=\left\{
\begin{aligned}
&\frac{\lambda_b-\lambda_s\|\tilde{s}_j\|_0}{\left|M_{j}(\tilde{B})\right|-\|\tilde{s}_j\|_0}\sgn\left(\tilde{b}_{j}^{(k)}\right) \\ &\qquad\qquad\qquad k\in M_{j}(\tilde{B})\quad\&\quad (j,k)\notin\support(\tilde{S})\\
&\\
&0\qquad\qquad\qquad\qquad\qquad\qquad\text{ow}
\end{aligned}\right..\\
\nonumber
\end{equation}

\noindent To get an explicit form for $\EstDual_{\bigcap_{k=1}^r\mathcal{U}_k^c}$, let $\delpar = \tilde{B}+\tilde{S}-\Bst-\Ss$. From the optimality conditions for the oracle problem~\eqref{EqnOracleEst}, we have
\begin{equation}
\frac{1}{n}\tr{X^{(k)}_{\mathcal{U}_k}}{X^{(k)}_{\mathcal{U}_k}} \delpar^{(k)}_{\mathcal{U}_k} - \frac{1}{n} \left(X^{(k)}_{\mathcal{U}_k}\right)^T w^{(k)} + \tilde{z}^{(k)}_{\mathcal{U}_k} = 0.
\nonumber
\end{equation}
and consequently,
\begin{equation}
\delpar^{(k)}_{\mathcal{U}_k}=\left(\frac{1}{n}\tr{X^{(k)}_{\mathcal{U}_k}}{X^{(k)}_{\mathcal{U}_k}}\right)^{-1} \left(\frac{1}{n} \left(X^{(k)}_{\mathcal{U}_k}\right)^T w^{(k)} - \tilde{z}^{(k)}_{\mathcal{U}_k}\right).
\label{errmag}
\end{equation}

\noindent Solving for $\tilde{z}^{(k)}_{\bigcap_{k=1}^r\mathcal{U}^c_k}$, for all $j\in{\bigcap_{k=1}^r\mathcal{U}_k^c}$, we get
\begin{equation}
\tilde{z}^{(k)}_{j} = - \frac{1}{n}\tr{X^{(k)}_{j}}{X^{(k)}_{\mathcal{U}_k}} \delpar^{(k)}_{\mathcal{U}_k} + \frac{1}{n} \left(X^{(k)}_{j}\right)^T w^{(k)}.
\nonumber
\end{equation}
Substituting for the value of $\delpar^{(k)}_{\mathcal{U}_k}$, we get
\begin{equation}\label{EqnDualBarAssign}
\begin{aligned}
\tilde{z}^{(k)}_{j}\!\! &= \frac{1}{n} \left(X^{(k)}_{j}\right)^T\!\!\!w^{(k)} \!\!-\!\frac{1}{n}\tr{X^{(k)}_{j}}{X^{(k)}_{\mathcal{U}_k}}\! \left(\frac{1}{n}\tr{X^{(k)}_{\mathcal{U}_k}}{X^{(k)}_{\mathcal{U}_k}}\right)^{\!-1}\\ &\qquad\qquad\qquad\qquad\qquad\qquad\left(\frac{1}{n} \left(X^{(k)}_{\mathcal{U}_k}\right)^T w^{(k)} - \tilde{z}^{(k)}_{\mathcal{U}_k}\right).\\
\end{aligned}
\end{equation}

{\bf STEP 2:}
This step consists of showing that the pair $(\tilde{S},\tilde{B},\EstDual)$ constructed in the earlier step is actually a \emph{feasible} primal-dual pair of \eqref{EqnDirtyEstNoisy}.
This would then the required support-recovery result since the constructed primal candidate $\tilde{S},\tilde{B}$ had the required sparsity pattern by \emph{construction}.

We will make use of the following lemma that specifies a set of sufficient (stationary) optimality conditions for the $(\tilde{S},\tilde{B})$ from \eqref{EqnOracleEst} to be the unique solution of the (unrestricted) optimization problem~\eqref{EqnDirtyEstNoisy}:

\begin{lemma}\label{LemWitnessOptCond}
Under our (stationary) assumptions on the design matrices $X^{(k)}$, the matrix pair $(\tilde{S},\tilde{B})$ is the unique solution of the problem \eqref{EqnDirtyEstNoisy} if there exists a matrix $\EstDual\in\mathbb{R}^{p\times r}$ such that 
\begin{itemize}
\item [(C1)] $P_{U_s}(\EstDual)=\lambda_s\sgn\left(\tilde{S}\right)$.\\

\item [(C2)] $P_{U_b}(\EstDual)=\left\{ \begin{array}{cc} t_{j}^{(k)} \, \sgn\left(\tilde{b}^{(k)}_{j}\right), & k \in M_{j}(\Bst)\\ 0 & \text{o.w.}.
\end{array}\right.$, where, $t_{j}^{(k)} \ge 0$ such that $\sum_{k \in M_j(\Bst)} t_j^{(k)} = \lambda_b$.\\

\item [(C3)] $\left\|P_{U_s^c}(\EstDual)\right\|_{\infty,\infty} < \lambda_s$.\\

\item [(C4)] $\left\|P_{U_b^c}(\EstDual)\right\|_{\infty,1} < \lambda_b$.\\

\item [(C5)] $\frac{1}{n} \tr{X^{(k)}}{X^{(k)}} \left(\tilde{b}^{(k)}\!\!+\!\tilde{s}^{(k)}\right)\!\! -\! \frac{1}{n} (X^{(k)})^T\! y^{(k)}\!\! +\! \tilde{z}^{(k)}\! =\! 0$ $\quad\forall 1\leq k\leq r$.\\
\end{itemize}
\end{lemma}

\begin{proof}
By assumptions (C1) and (C3), $\frac{1}{\lambda_s}\EstDual\in\partial\|\tilde{S}\|_{1,1}$ and by assumptions (C2) and (C4), $\frac{1}{\lambda_b}\EstDual\in\partial\|\tilde{B}\|_{1,\infty}$. Thus, $(\tilde{S},\tilde{B},\EstDual)$ is a feasible primal-dual pair of \eqref{EqnDirtyEstNoisy} according to the Lemma~\ref{LemOptNoisy}.\\

\noindent Let $\mathbb{B}$ and $\mathbb{S}$ to be balls of $\ell_\infty/\ell_1$ and $\ell_\infty/\ell_\infty$ with radiuses $\lambda_b$ and $\lambda_s$, respectively. Considering the fact that $\lambda_b\|B\|_{1,\infty}=\sup_{Z\in\mathbb{B}}\tr{Z}{B}$ and $\lambda_s\|S\|_{1,1}=\sup_{Z\in\mathbb{S}}\tr{Z}{S}$, the problem~\eqref{EqnDirtyEstNoisy} can be written as
\begin{equation}
\begin{aligned}
(\Sh,\Bh)=\arg\inf_{S,B}\sup_{Z\in\mathbb{B}\cap\mathbb{S}} \Bigg\{&\frac{1}{2n}\sum_{k=1}^{r}\left\|y^{(k)}\!\!-\!X^{(k)}\left(b^{(k)}\!\!+\!s^{(k)}\right)\right\|_2^2\\ &\,\,\,\quad\qquad\qquad+\tr{Z}{S}+\tr{Z}{B}\Bigg\}.
\end{aligned}
\nonumber
\end{equation}
This saddle-point problem is strictly feasible and convex-concave. Given any dual variable, in particular $\EstDual$, and any primal optimal $(\Sh,\Bh)$ we have $\lambda_b\|\Bh\|_{1,\infty}=\tr{\EstDual}{\Bh}$ and $\lambda_s\|\Sh\|_{1,1}=\tr{\EstDual}{\Sh}$. This implies that $\1b=\mathbf{0}$ if $\left\|\tilde{z}_j\right\|_{1}<\lambda_b$ (because $\lambda_b\sum_j\|\1b\|_\infty\leq\sum_j\left\|\tilde{z}_j\right\|_{1}\|\1b\|_\infty$ and if $\|\tilde{z}_{j_0}\|_1<\lambda_b$ for some $j_0$, then others can not compensate for that in the sum due to the fact that $\EstDual\in\mathbb{B}$, i.e., $\|\tilde{z}_j\|_1\leq\lambda_b$). It also implies that $\Bs=0$ if $\left|\tilde{z}_j^{(k)}\right|<\lambda_s$ for a similar reason. Hence, $P_{U_b^c}(\hat{B})=0$ and $P_{U_s^c}(\hat{S})=0$. This means that solving the restricted problem~\eqref{EqnOracleEst} is equivalent to solving the problem~\eqref{EqnDirtyEstNoisy}.\\

The uniqueness follows from our (stationary) assumptions on design matrices $X^{(k)}$ that the matrix $\frac{1}{n}\tr{X^{(k)}_{\mathcal{U}_k}}{X^{(k)}_{\mathcal{U}_k}}$ is invertible for all $1\leq k\leq r$. Using this assumption, the problem~\eqref{EqnOracleEst} is \emph{strictly} convex and the solution is unique. Consequently, the solution of \eqref{EqnDirtyEstNoisy} is also unique, since we showed that these two problems are equivalent. This concludes the proof of the lemma.\\
\end{proof}
\vspace{0.5cm}

By construction, the primal-dual pair $(\tilde{B},\tilde{S},\EstDual)$  satisfies the (C1), (C2) and (C5) conditions in Lemma~\ref{LemWitnessOptCond}. \emph{It only remains to guarantee (C3) and (C4) separately for each of the theorems.} 

\noindent Indeed, this is where the proofs of the theorems differ. Specifically, Lemmas~\ref{problemma}, \ref{LemmaGaussDualProj} and \ref{ThresholdCertificateLemma} ensure these conditions are satisfied with given sample complexities in Theorems 1, 2 and 3, respectively.

\section{Proofs} 

The proofs of our three main theorems are in sections~\ref{appendix:proof_thm1}, \ref{appendix:proof_thm2} and \ref{appendix:proof_thm3} respectively.

\subsection{Proof of Theorem 1}
\label{appendix:proof_thm1}
Let $d=\lfloor\frac{\lambda_b}{\lambda_s}\rfloor$ and $(\Bst,\Ss)=\mathcal{H}_d(\bar{\Theta})$. Then, the result follows from Proposition~\ref{Th1} below.\\

\begin{proposition}[Structure Recovery]\label{Th1}
Under assumptions of Theorem 1, with probability $1-c_1\exp(-c_2n)$ for some positive constants $c_1$ and $c_2$, we are guaranteed that the following properties hold:\\

\begin{itemize}
\item [(P1)] Problem \eqref{EqnDirtyEstNoisy} has unique solution $(\Sh,\Bh)$ such that $\support(\Sh)\subseteq\support(\Ss)$ and $\rowsupport(\Bh)\subseteq\rowsupport(\Bst)$.\\

\item [(P2)] $\left\|\Bh+\Sh-\Bst-\Ss\right\|_\infty\leq\underbrace{\sqrt{\frac{4\sigma^2\log\left(pr\right)}{C_{min}n}}+\lambda_sD_{max}}_{b_{\min}}$.\\

\item [(P3)] $\sgn\left(\support(\As)\right)=\sgn\left(\support(\ssj)\right)$ 

for all $j\notin\rowsupport(\Bst)$ provided that $$\displaystyle\min_{j\notin\rowsupport(\Bst)\atop{(j,k)\in\support(\Ss)}}  \left|\ssjk\right|>b_{\min}.$$\\

\item [(P4)] $\sgn\left(\support(\As+\1b)\right)=\sgn\left(\support(\ssj+\bsj)\right)$  

for all $j\in\rowsupport(\Bst)$ provided that $$\displaystyle\min_{(j,k)\in\support(\Bst)}\left|\bsjk+\ssjk\right|>b_{\min}.$$\\
\end{itemize}
\end{proposition}

\begin{proof}
We prove the result separately for each part.
\begin{itemize}
\item [(P1)] Considering the constructed primal-dual pair, it suffices to show that (C3) and (C4) in Lemma~\ref{LemWitnessOptCond} are satisfied with high probability. By Lemma~\ref{problemma}, with probability at least $1-c_1\exp(-c_2n)$ those two conditions hold and hence, $(\Sh,\Bh)=(\tilde{S},\tilde{B})$ is the unique solution of (\ref{EqnDirtyEstNoisy}) and the property (P1) follows.\\

\item [(P2)] Using (\ref{errmag}), we have
\footnotesize\begin{equation}
\begin{aligned}
\max_{j\in\mathcal{U}_k}\left|\delpar_j^{(k)}\right|& \leq\left\|\left(\frac{1}{n}\tr{X^{(k)}_{\mathcal{U}_k}}{X^{(k)}_{\mathcal{U}_k}}\right)^{\!\!-1} \!\!\frac{1}{n}\! \left(X^{(k)}_{\mathcal{U}_k}\right)^T\!\!\! w^{(k)}\right\|_\infty\\ &\qquad\qquad+ \left\|\left(\frac{1}{n}\tr{X^{(k)}_{\mathcal{U}_k}}{X^{(k)}_{\mathcal{U}_k}}\right)^{-1}\tilde{z}^{(k)}_{\mathcal{U}_k}\right\|_\infty\\
&\leq\sqrt{\frac{4\sigma^2\log\left(pr\right)}{C_{min}n}}+\lambda_sD_{max},
\end{aligned}
\nonumber
\end{equation}\normalsize
where, the second inequality holds with high probability as a result of Lemma~\ref{NoiseInf} for $\alpha=\epsilon\sqrt{\frac{4\sigma^2\log\left(pr\right)}{C_{min}n}}$ for some $\epsilon>1$, considering the fact that $\text{Var}\left(\delpar_j^{(k)}\right)\leq\frac{\sigma^2}{C_{min}n}$.\\

\item [(P3)] Using (P1) in Lemma~\ref{signBSnoisy}, this event is equivalent to the event that for all $j\notin\rowsupport(\Bst)$ with $(j,k)\in\support(\Ss)$, we have $\left(\delpar_j^{(k)}+\ssjk\right)\sgn\left(\ssjk\right)>0$. By Hoeffding inequality, we have
\small\begin{equation}
\begin{aligned}
&\mathbb{P}\left[\left(\delpar_j^{(k)}+\ssjk\right)\sgn\left(\ssjk\right)>0\right]\\ &\qquad\qquad\qquad\qquad=\mathbb{P}\Bigg[-\delpar_j^{(k)}\sgn\left(\ssjk\right)<\left|\ssjk\right|\Bigg]\\
&\qquad\qquad\qquad\qquad\geq \mathbb{P}\Bigg[\left|\delpar_j^{(k)}\right|<\left|\ssjk\right|\Bigg].
\end{aligned}
\nonumber
\end{equation}\normalsize
By part (P2), this event happens with high probability if $\displaystyle\min_{j\notin\rowsupport(\Bst)\atop{(j,k)\in\support(\Ss)}} \left|\ssjk\right|>b_{\min}$.\\

\item [(P4)] Using (P1) in Lemma~\ref{signBSnoisy}, this event is equivalent to the event that for all $j\in\rowsupport(\Bst)$, we have $\left(\delpar_j^{(k)}+\bsjk+\ssjk\right)\sgn\left(\bsjk+\ssjk\right)>0$. By Hoeffding inequality, we have
\small\begin{equation}
\begin{aligned}
&\mathbb{P}\left[\left(\delpar_j^{(k)}+\bsjk+\ssjk\right)\sgn\left(\bsjk+\ssjk\right)>0\right]\\ &\qquad=\mathbb{P}\Bigg[-\delpar_j^{(k)}\sgn\left(\bsjk+\ssjk\right)<\left|\bsjk+\ssjk\right|\Bigg]\\
&\qquad\geq \mathbb{P}\Bigg[\left|\delpar_j^{(k)}\right|<\left|\bsjk+\ssjk\right|\Bigg].
\end{aligned}
\nonumber
\end{equation}\normalsize
By part (P2), this event happens with high probability if $\displaystyle\min_{(j,k)\in\support(\Bst)}\left|\bsjk+\ssjk\right|>b_{\min}$.
\end{itemize}
\end{proof}

\begin{lemma}
Under conditions of Proposition~\ref{Th1}, the conditions (C3) and (C4) in Lemma~\ref{LemWitnessOptCond} hold for the constructed primal-dual pair with probability at least $1-c_1\exp(-c_2n)$ for some positive constants $c_1$ and $c_2$.
\label{problemma}
\end{lemma}

\begin{proof}
First, we need to bound the projection of $\EstDual$ into the space $U_s^c$. Notice that
\footnotesize\begin{equation}
\left|\left(P_{U_s^c}(\EstDual)\right)_j^{(k)}\right|=\left\{
\begin{aligned}
&\frac{\lambda_b-\lambda_s\|\tilde{s}_j\|_0}{\left|M_{j}(\tilde{B})\right|-\|\tilde{s}_j\|_0}\\ &\quad j\in\rowsupport(\tilde{B})\;\&\; (j,k)\notin\support(\tilde{S})\\
&\\
&\left|\tilde{z}^{(k)}_{j}\right|\qquad\qquad j\in\bigcap_{k=1}^r\mathcal{U}_{k}^c\\
&\\
&0\qquad\qquad\qquad\qquad\text{ow}.\\
\end{aligned}\right..
\nonumber
\end{equation}\normalsize
By our assumption on the ratio of the penalty regularizer coefficients, we have $\frac{\lambda_b-\lambda_s\|\tilde{s}_j\|_0}{\left|M_{j}(\tilde{B})\right|-\|\tilde{s}_j\|_0}<\lambda_s$. Moreover, we have
\small\begin{equation}
\begin{aligned}
\left|\tilde{z}^{(k)}_{j}\right|&\leq\max_{j\in\bigcap_{k=1}^r\mathcal{U}_k^c} \left\|\frac{1}{n}\tr{X^{(k)}_{j}}{X^{(k)}_{\mathcal{U}_k}}\left(\frac{1}{n}\tr{X^{(k)}_{\mathcal{U}_k}}{X^{(k)}_{\mathcal{U}_k}}\right)^{-1}\right\|_1\\ &\qquad\qquad\qquad\qquad\qquad\left(\left\|\frac{1}{n} \left(X^{(k)}\right)^T w^{(k)}\right\|_{\infty}\!\!\!+\! \left\|\tilde{z}^{(k)}_{\mathcal{U}_k}\right\|_{\infty}\right)\\ &\qquad+\! \left\|\frac{1}{n} \left(X^{(k)}\right)^T w^{(k)}\right\|_{\infty}\\
&\leq(2-\gamma_s)\left\|\frac{1}{n} \left(X^{(k)}\right)^T w^{(k)}\right\|_{\infty}+(1-\gamma_s)\left\|\tilde{z}^{(k)}_{\mathcal{U}_k}\right\|_{\infty}\\
&\leq(2-\gamma_s)\left\|\frac{1}{n} \left(X^{(k)}\right)^T w^{(k)}\right\|_{\infty}+(1-\gamma_s)\lambda_s.\\
\end{aligned}
\nonumber
\end{equation}\normalsize
Thus, the event $\|P_{U_s^c}(\EstDual)\|_{\infty,\infty}<\lambda_s$ is equivalent to the event \footnotesize$\displaystyle\max_{1\leq k\leq r}\left\|\frac{1}{n}\left(X^{(k)}\right)^T w^{(k)}\right\|_{\infty}<\frac{\gamma_s}{2-\gamma_s}\lambda_s\,\,$\normalsize. By Lemma~\ref{NoiseInf}, this event happens with probability at least $1-2\exp\left(-\frac{\gamma_s^2n\lambda_s^2}{4(2-\gamma_s)^2\sigma^2}+\log(pr)\right)$. This probability goes to $1$ if $\lambda_s>\frac{2(2-\gamma_s)\sigma\sqrt{\log(pr)}}{\gamma_s\sqrt{n}}$ as stated in the assumptions.\\

\noindent
Next, we need to bound the projection of $\EstDual$ into the space $U_b^c$. Notice that
\small\begin{equation}
\sum_{k=1}^r\left|\left(P_{U_b^c}(\EstDual)\right)_j^{(k)}\right|=\left\{
\begin{aligned}
&\lambda_s\|\tilde{s}_j\|_0&j\in\bigcup_{k=1}^r\mathcal{U}_k-\rowsupport(\Bst)\\
&\sum_{k=1}^r\left|\tilde{z}^{(k)}_{j}\right|&j\in\bigcap_{k=1}^r\mathcal{U}_k^c\\
&0&\text{ow}
\end{aligned}\right..
\nonumber
\end{equation}\normalsize
We have $\lambda_s\|\tilde{s}_j\|_0\leq\lambda_s D(\Ss)<\lambda_b$ by our assumption on the ratio of the penalty regularizer coefficients. We can establish the following bound:
\footnotesize\begin{equation}
\begin{aligned}
&\sum_{k=1}^{r}\left|\tilde{z}^{(k)}_{j}\right|\\ &\leq \max_{j\in\bigcap_{k=1}^r\mathcal{U}_k^c}\sum_{k=1}^{r}\left\|\frac{1}{n}\tr{X^{(k)}_{j}}{X^{(k)}_{\mathcal{U}_k}}\left(\frac{1}{n}\tr{X^{(k)}_{\mathcal{U}_k}}{X^{(k)}_{\mathcal{U}_k}}\right)^{-1}\right\|_1\\ &\qquad\qquad\qquad\left(\max_{j\in\bigcup_{k=1}^r\mathcal{U}_k}\left\|\tilde{z}^{(k)}_{j}\right\|_{1}+\max_{1\leq k\leq r}\left\|\frac{1}{n}\left(X^{(k)}\right)^T w^{(k)}\right\|_{\infty}\right)\\
&\qquad+\max_{1\leq k\leq r}\left\|\frac{1}{n}\left(X^{(k)}\right)^T w^{(k)}\right\|_{\infty}\\
&\leq(1-\gamma_b)\lambda_b+(2-\gamma_b)\max_{1\leq k\leq K}\left\|\frac{1}{n}\left(X^{(k)}\right)^T w^{(k)}\right\|_{\infty}.\\
\end{aligned}
\nonumber
\end{equation}\normalsize
Thus, the event $\|P_{U_b^c}(\EstDual)\|_{\infty,1}<\lambda_b$ is equivalent to the event
\small$\max_{1\leq k\leq r}\left\|\frac{1}{n}\left(X^{(k)}\right)^T w^{(k)}\right\|_{\infty}<\frac{\gamma_b}{2-\gamma_b}\lambda_b\,$\normalsize. By Lemma~\ref{NoiseInf}, this event happens with probability at least $1-2\exp\left(-\frac{\gamma_b^2n\lambda_b^2}{4(2-\gamma_b)^2\sigma^2}+\log(pr)\right)$. This probability goes to $1$ if $\lambda_b>\frac{2(2-\gamma_b)\sigma\sqrt{\log(pr)}}{\gamma_b\sqrt{n}}$ as stated in the assumptions.\\

Hence, with probability at least $1-c_1\exp(-c_2n)$ conditions (C3) and (C4) in Lemma~\ref{LemWitnessOptCond} are satisfied.\\
\end{proof}

\begin{lemma}
\footnotesize$$\mathbb{P}\left[\displaystyle\max_{1\leq k\leq r}\left\|\frac{1}{n}\left(X^{(k)}\right)^T w^{(k)}\right\|_{\infty}<\alpha\right]\geq 1-2\exp\left(\displaystyle-\frac{\alpha^2n}{4\sigma^2}+\log(pr)\right).$$\normalsize
\label{NoiseInf}
\end{lemma}

\begin{proof}
Since $w_j^{(k)}$'s are distributed as $\mathcal{N}(0,\sigma^2)$, we have $\frac{1}{n}\left(X^{(k)}\right)^T w^{(k)}$ distributed as $\mathcal{N}\left(0,\frac{\sigma^2}{n}\left(X^{(k)}\right)^{T}X^{(k)}_{\mathcal{U}_k}\right)$. Using Hoeffding inequality, we have
\footnotesize\begin{equation}
\begin{aligned}
\mathbb{P}\left[\left\|\frac{1}{n}\left(X^{(k)}\right)^T w^{(k)}\right\|_{\infty}\geq\alpha\right]&\leq\sum_{j=1}^{p}\mathbb{P}\left[\left|\frac{1}{n}\left(X^{(k)}_j\right)^T w^{(k)}\right|\geq\alpha\right]\\
&\leq \sum_{j=1}^{p}2\exp\left(\displaystyle-\frac{\alpha^2n}{2\sigma^2\left(X^{(k)}_j\right)^{T}X^{(k)}_j}\right)\\
&\leq 2p\exp\left(\displaystyle-\frac{\alpha^2n}{4\sigma^2}\right).
\end{aligned}
\nonumber
\end{equation}\normalsize
By union bound, the result follows.\\
\end{proof}

\vspace{0.25cm}
\subsection{Proof of Theorem 2}
\label{appendix:proof_thm2}
Let $d=\lfloor\frac{\lambda_b}{\lambda_s}\rfloor$ and $(\Bst,\Ss)=\mathcal{H}_d(\bar{\Theta})$. Then, the result follows from the next proposition.\\

\begin{proposition} Under assumptions of Theorem 2, if $$n>\max\left(\frac{Bs\log(pr)}{C_{min}\gamma_s^2},\frac{Bsr\big(r\log(2)+\log(p)\big)}{C_{min}\gamma_b^2}\right)$$ then with probability at least $1-c_1\exp\left(-c_2\left(r\log(2)+\log(p)\right)\right)-c_3\exp(-c_4\log(rs))$ for some positive constants $c_1-c_4$, we are guaranteed that the following properties hold:\\

\begin{itemize}
\item [(P1)] The solution $(\Bh,\Sh)$ to (\ref{EqnDirtyEstNoisy}) is unique and $\rowsupport(\Bh)\subseteq\rowsupport(\Bst)$ and $\support(\Sh)\subseteq\support(\Ss)$.\\

\item [(P2)] \footnotesize$\left\|\Bh+\Sh-\Bst-\Ss\right\|_{\infty}\leq \underbrace{\sqrt{\frac{50\sigma^2\log(rs)}{nC_{min}}}+\lambda_s\left(\frac{Ds}{C_{min}\sqrt{n}}+D_{max}\right)}_{g_{\min}}\,\,$\normalsize.\\

\item [(P3)] $\sgn\left(\support(\As)\right)=\sgn\left(\support(\ssj)\right)$

for all $j\notin\rowsupport(\Bst)$ provided that $$\displaystyle\min_{j\notin\rowsupport(\Bst)\atop{(j,k)\in\support(\Ss)}}  \left|\ssjk\right|>g_{\min}.$$\\

\item [(P4)] $\sgn\left(\support(\As+\1b)\right)=\sgn\left(\support(\ssj+\bsj)\right)$

for all $j\in\rowsupport(\Bst)$ provided that $$\displaystyle\min_{(j,k)\in\support(\Bst)}\left|\bsjk+\ssjk\right|>g_{\min}.$$
\end{itemize}
\label{GaussianProp}
\end{proposition}

\begin{proof}
We provide the proof of each part separately.
\begin{itemize}
\item [(P1)] Considering the constructed primal-dual pair $(\tilde{S},\tilde{B},\EstDual)$, it suffices to show that the conditions (C3) and (C4) in Lemma~\ref{LemWitnessOptCond} are satisfied under these assumptions. Lemma~\ref{LemmaGaussDualProj} guarantees that with probability at least $1-c_1\exp\left(-c_2\left(r\log(2)+\log(p)\right)\right)$ those conditions are satisfied. Hence, $(\Bh,\Sh)=(\tilde{B},\tilde{S})$ are the unique solution to (\ref{EqnDirtyEstNoisy}) and (P1) follows.\\

\item [(P2)] From (\ref{errmag}), we have
\footnotesize\begin{equation}
\begin{aligned}
\max_{j\in\mathcal{U}_k}\left|\delpar_j^{(k)}\right| &\leq\underbrace{\left\|\left(\frac{1}{n}\tr{X^{(k)}_{\mathcal{U}_k}}{X^{(k)}_{\mathcal{U}_k}}\right)^{-1} \frac{1}{n} \left(X^{(k)}_{\mathcal{U}_k}\right)^T w^{(k)}\right\|_\infty}_{\mathcal{W}^{(k)}}\\ &\qquad\qquad+ \left\|\left(\frac{1}{n}\tr{X^{(k)}_{\mathcal{U}_k}}{X^{(k)}_{\mathcal{U}_k}}\right)^{-1}\tilde{z}^{(k)}_{\mathcal{U}_k}\right\|_\infty\\
&\leq\left\|\mathcal{W}^{(k)}\right\|_\infty +\left\|\left(\Sigma^{(k)}_{\mathcal{U}_k,\mathcal{U}_k}\right)^{-1}\tilde{z}^{(k)}_{\mathcal{U}_k}\right\|_\infty\\
&\quad+\left\|\left(\left(\frac{1}{n}\tr{X^{(k)}_{\mathcal{U}_k}}{X^{(k)}_{\mathcal{U}_k}}\right)^{-1}\!\!\!\!-\left(\Sigma^{(k)}_{\mathcal{U}_k,\mathcal{U}_k}\right)^{-1}\right)\tilde{z}^{(k)}_{\mathcal{U}_k}\right\|_\infty.\\ 
\end{aligned}
\nonumber
\end{equation}\normalsize
\noindent We need to bound these three quantities. Notice that
\begin{equation}
\begin{aligned}
\left\|\left(\Sigma^{(k)}_{\mathcal{U}_k,\mathcal{U}_k}\right)^{-1}\tilde{z}^{(k)}_{\mathcal{U}_k}\right\|_\infty &\leq\left\|\left(\Sigma^{(k)}_{\mathcal{U}_k,\mathcal{U}_k}\right)^{-1}\right\|_{\infty,1}\left\|\tilde{z}^{(k)}_{\mathcal{U}_k}\right\|_\infty\\
&\leq D_{max}\lambda_s.
\end{aligned}
\nonumber
\end{equation}
Also, we have
\footnotesize\begin{equation}
\begin{aligned}
&\left\|\left(\left(\frac{1}{n}\tr{X^{(k)}_{\mathcal{U}_k}}{X^{(k)}_{\mathcal{U}_k}}\right)^{-1}\!\!\!\!-\left(\Sigma^{(k)}_{\mathcal{U}_k,\mathcal{U}_k}\right)^{-1}\right)\tilde{z}^{(k)}_{\mathcal{U}_k}\right\|_\infty\\ 
&\qquad\leq\lambda_{max}\left(\left(\frac{1}{n}\tr{X^{(k)}_{\mathcal{U}_k}}{X^{(k)}_{\mathcal{U}_k}}\right)^{-1}\!\!\!\!-\left(\Sigma^{(k)}_{\mathcal{U}_k,\mathcal{U}_k}\right)^{-1}\right)\left\|\tilde{z}^{(k)}_{\mathcal{U}_k}\right\|_2\\
&\qquad\leq\lambda_{max}\left(\left(\frac{1}{n}\tr{X^{(k)}_{\mathcal{U}_k}}{X^{(k)}_{\mathcal{U}_k}}\right)^{-1}\!\!\!\!-\left(\Sigma^{(k)}_{\mathcal{U}_k,\mathcal{U}_k}\right)^{-1}\right)\sqrt{s}\lambda_s\\
&\qquad\leq\frac{4}{C_{min}}\sqrt{\frac{s}{n}}\sqrt{s}\lambda_s,
\end{aligned}
\nonumber
\end{equation}\normalsize
where, the last inequality holds with probability at least $1-c_1\exp\left(-c_2\left(\sqrt{n}-\sqrt{s}\right)^2\right)$ for some positive constants $c_1$ and $c_2$ as a result of \cite{DAVSZA} on eigenvalues of Gaussian random matrices. Conditioned on $X_{\mathcal{U}_k}^{(k)}$, the vector $\mathcal{W}^{(k)}\in\mathbb{R}^{\left|\mathcal{U}_k\right|}$ is a zero-mean Gaussian random vector with covariance matrix $\frac{\sigma^2}{n}\left(\frac{1}{n}\tr{X^{(k)}_{\mathcal{U}_k}}{X^{(k)}_{\mathcal{U}_k}}\right)^{-1}$. Thus, we have
\footnotesize\begin{equation}
\begin{aligned}
&\frac{1}{n}\lambda_{max}\left(\left(\frac{1}{n}\tr{X^{(k)}_{\mathcal{U}_k}}{X^{(k)}_{\mathcal{U}_k}}\right)^{-1}\right)\\ &\qquad\leq\frac{1}{n}\lambda_{max}\left(\left(\frac{1}{n}\tr{X^{(k)}_{\mathcal{U}_k}}{X^{(k)}_{\mathcal{U}_k}}\right)^{-1}\!\!\!\!-\left(\Sigma^{(k)}_{\mathcal{U}_k,\mathcal{U}_k}\right)^{-1}\right)\\ &\qquad\qquad\qquad\qquad\qquad\qquad\qquad\qquad+\frac{1}{n}\lambda_{max}\left(\left(\Sigma^{(k)}_{\mathcal{U}_k,\mathcal{U}_k}\right)^{-1}\right)\\
&\qquad\leq\frac{1}{n}\left(\frac{4}{C_{min}}\sqrt{\frac{s}{n}}+\frac{1}{C_{min}}\right)\\
&\qquad\leq\frac{5}{nC_{min}}.
\end{aligned}
\nonumber
\end{equation}\normalsize
\noindent From the concentration of Gaussian random variables (Lemma~\ref{NoiseInf}) and using the union bound, we get
\small\begin{equation}
\mathbb{P}\left[\max_{1\leq k\leq r}\left\|\mathcal{W}^{(k)}\right\|_\infty\geq t\right]\leq 2\exp\left(-\frac{t^2nC_{min}}{50\sigma^2}+\log(rs)\right).\\
\nonumber
\end{equation}\normalsize
For $t=\epsilon\sqrt{\frac{50\sigma^2\log(rs)}{nC_{min}}}$ for some $\epsilon>1$, the result follows.\\
\item [(P3),(P4)] The results are immediate consequence of (P2).
\end{itemize}
\end{proof}

\begin{lemma}
Under the assumptions of Proposition~\ref{GaussianProp}, the conditions (C3) and (C4) in Lemma~\ref{LemWitnessOptCond} hold for the constructed primal-dual pair with probability at least $1-c_1\exp\left(-c_2\left(r\log(2)+\log(p)\right)\right)$ for some positive constants $c_1$ and $c_2$.
\label{LemmaGaussDualProj}
\end{lemma}

\begin{proof}
First, we need to bound the projection of $\EstDual$ into the space $U_s^c$. Notice that
\footnotesize\begin{equation}
\left|\left(P_{U_s^c}(\EstDual)\right)_j^{(k)}\right|=\left\{
\begin{aligned}
&\frac{\lambda_b-\lambda_s\|\tilde{s}_j\|_0}{\left|M_{j}(\tilde{B})\right|-\|\tilde{s}_j\|_0}\\ &\qquad\qquad j\in\rowsupport(\tilde{B})\quad\&\quad (j,k)\notin\support(\tilde{S})\\
&\left|\tilde{z}^{(k)}_{j}\right|\qquad\qquad\qquad\qquad j\in\bigcap_{k=1}^r\mathcal{U}_{k}^c\\
&0\qquad\qquad\qquad\qquad\qquad\qquad\text{ow}.\\
\end{aligned}\right..
\nonumber
\end{equation}\normalsize
By our assumptions on the ratio of the penalty regularizer coefficients, we have $\frac{\lambda_b-\lambda_s\|\tilde{s}_j\|_0}{\left|M_{j}(\tilde{B})\right|-\|\tilde{s}_j\|_0}<\lambda_s$. For all $j\in\bigcap_{k=1}^r\mathcal{U}_k$ and $R\in\mathbb{R}^{p\times r}$ with i.i.d. standard Gaussian entries (see Lemma 4 in \cite{NWJoint}), we have
\footnotesize\begin{equation}
\begin{aligned}
&\left|\tilde{z}^{(k)}_{j}\right|\\ 
&\leq\!\!\!\!\!\!\max_{j\in\bigcap_{k=1}^r\mathcal{U}_k^c} \underbrace{\left|\frac{1}{n}\tr{X^{(k)}_{j}\!\!}{\mathbf{I}-\frac{1}{n}X^{(k)}_{\mathcal{U}_k}\left(\frac{1}{n}\tr{X^{(k)}_{\mathcal{U}_k}\!\!}{X^{(k)}_{\mathcal{U}_k}}\right)^{\!\!-1} \!\!\!\!\!\left(X^{(k)}_{\mathcal{U}_k}\right)^T} w^{(k)}\right|}_{\mathcal{W}_j^{(k)}}\\
&\qquad+\max_{j\in\bigcap_{k=1}^r\mathcal{U}_k^c} \left|\frac{1}{n}\tr{X^{(k)}_{j}}{X^{(k)}_{\mathcal{U}_k}\left(\frac{1}{n}\tr{X^{(k)}_{\mathcal{U}_k}}{X^{(k)}_{\mathcal{U}_k}}\right)^{-1}} \tilde{z}^{(k)}_{\mathcal{U}_k}\right|\\ &\leq\max_{j\in\bigcap_{k=1}^r\mathcal{U}_k^c}\left|\mathcal{W}_j^{(k)}\right|+\max_{j\in\bigcap_{k=1}^r\mathcal{U}_k^c} \left\|\Sigma^{(k)}_{j,\mathcal{U}_k}\left(\Sigma^{(k)}_{\mathcal{U}_k,\mathcal{U}_k}\right)^{-1}\right\|_1 \left\|\tilde{z}^{(k)}_{\mathcal{U}_k}\right\|_{\infty}\\
&\quad+\max_{j\in\bigcap_{k=1}^r\mathcal{U}_k^c} \underbrace{\left|\frac{1}{n}\tr{R^{(k)}_{j}}{X^{(k)}_{\mathcal{U}_k}\left(\frac{1}{n}\tr{X^{(k)}_{\mathcal{U}_k}}{X^{(k)}_{\mathcal{U}_k}}\right)^{-1}} \tilde{z}^{(k)}_{\mathcal{U}_k}\right|}_{\mathcal{R}_j^{(k)}}\\
&\leq(1-\gamma_s)\lambda_s+\max_{j\in\bigcap_{k=1}^r\mathcal{U}_k^c}\left|\mathcal{R}_j^{(k)}\right|+\max_{j\in\bigcap_{k=1}^r\mathcal{U}_k^c}\left|\mathcal{W}_j^{(k)}\right|,
\end{aligned}
\nonumber
\end{equation}\normalsize
The second inequality follows from the triangle inequality on the distributions. By Lemma~\ref{l2bound}, if $n\geq\frac{2}{2-\sqrt{3}}\log(pr)$ then with high probability $\left\|X_j^{(k)}\right\|_2^2\leq 2n$ and hence $\text{Var}\left(\mathcal{W}_j^{(k)}\right)\leq\frac{2\sigma^2}{n}$. Using the concentration results for the zero-mean Gaussian random variable $\mathcal{W}_j^{(k)}$ and using the union bound, we get
\small\begin{equation}
\mathbb{P}\left[\max_{j\in\bigcap_{k=1}^r\mathcal{U}_k^c}\left|\mathcal{W}_j^{(k)}\right|\geq t\right]\leq 2\exp\left(-\frac{t^2n}{4\sigma^2}+\log(p)\right)\qquad\forall t\geq 0.\\
\nonumber
\end{equation}\normalsize

\noindent Conditioning on $\left(X^{(k)}_{\mathcal{U}_k},w^{(k)},\tilde{z}^{(k)}\right)$'s, we have that $\mathcal{R}_j^{(k)}$ is a zero-mean Gaussian random variable with
\begin{equation}
\text{Var}\left(\mathcal{R}_j^{(k)}\right)\leq\frac{\left\|\tilde{z}_{\mathcal{U}_k}^{(k)}\right\|_2^2}{nC_{min}}\leq\frac{s\lambda_s^2}{nC_{min}}.
\nonumber
\end{equation}
By concentration of Gaussian random variables, we have
\small\begin{equation}
\mathbb{P}\left[\max_{j\in\bigcap_{k=1}^r\mathcal{U}_k^c}\left|\mathcal{R}_j^{(k)}\right|\geq t\right]\leq 2\exp\left(-\frac{t^2nC_{min}} {Bs\lambda_s^2}+\log(p)\right)\quad\forall t\geq 0.\\
\nonumber
\end{equation}\normalsize
Using these bounds, we get
\footnotesize\begin{equation}
\begin{aligned}
&\!\mathbb{P}\!\left[\!\left\|P_{U_s^c}(\EstDual)\right\|_{\infty,\infty}\!\!\!<\!\lambda_s\right]\\ &\geq\mathbb{P}\left[\max_{j\in\bigcap_{k=1}^r\mathcal{U}_k^c}\left|\mathcal{R}_j^{(k)}\right|+\max_{j\in\bigcap_{k=1}^r\mathcal{U}_k^c}\left|\mathcal{W}_j^{(k)}\right|<\gamma_s\lambda_s\qquad\forall \,1\leq k\leq r\right]\\
&\geq\mathbb{P}\left[\max_{j\in\bigcap_{k=1}^r\mathcal{U}_k^c}\left|\mathcal{R}_j^{(k)}\right|<t_0\quad\forall\,1\leq k\leq r\right]\\ &\qquad\qquad\qquad\mathbb{P}\left[\max_{j\in\bigcap_{k=1}^r\mathcal{U}_k^c}\left|\mathcal{W}_j^{(k)}\right|<\gamma_s\lambda_s-t_0\quad\forall\,1\leq k\leq r\right]\\
&\geq\left(1-2\exp\left(-\frac{t_0^2nC_{min}}{Bs\lambda_s^2}+\log(pr)\right)\right)\\ &\qquad\qquad\qquad\left(1-2\exp\left(-\frac{(\gamma_s\lambda_s-t_0)^2n}{4\sigma^2}+\log(pr)\right)\right).
\end{aligned}
\nonumber
\end{equation}\normalsize
This probability goes to $1$ for $t_0=\frac{\sqrt{Bs}\lambda_s}{\sqrt{Bs}\lambda_s+2\sigma\sqrt{C_{min}}}\gamma_s\lambda_s$  (the solution to $\frac{t_0^2C_{min}} {Bs\lambda_s^2}=\frac{(\gamma_s\lambda_s-t_0)^2}{4\sigma^2}$), if the regularization parameter $\lambda_s>\frac{\sqrt{4\sigma^2C_{min}\log(pr)}}{\gamma_s\sqrt{nC_{min}}-\sqrt{Bs\log(pr)}}$ provided that $n>\frac{Bs\log(pr)}{C_{min}\gamma_s^2}$ as stated in the assumptions.
\vspace{1cm}

\noindent Next, we need to bound the projection of $\EstDual$ into the space $U_b^c$. Notice that
\footnotesize\begin{equation}
\sum_{k=1}^r\left|\left(P_{U_b^c}(\EstDual)\right)_j^{(k)}\right|=\left\{
\begin{aligned}
&\lambda_s\|\tilde{s}_j\|_0\qquad\quad&j\in\bigcup_{k=1}^r\mathcal{U}_k-\rowsupport(\Bst)\\
&\sum_{k=1}^{r}\left|\tilde{z}^{(k)}_{j}\right|&j\in\bigcap_{k=1}^r\mathcal{U}_k^c\\
&0&\text{ow}
\end{aligned}\right..
\nonumber
\end{equation}\normalsize
We have $\lambda_s\|\tilde{s}_j\|_0\leq\lambda_s D(\Ss)<\lambda_b$ by our assumption on the ratio of the penalty regularizer coefficients. For all $j\in\bigcap_{k=1}^r\mathcal{U}_k^c$, we have
\footnotesize\begin{equation}
\begin{aligned}
&\sum_{k=1}^{r}\left|\tilde{z}^{(k)}_{j}\right|\\ 
&\leq\!\!\!\!\!\!\max_{j\in\bigcap_{k=1}^r\mathcal{U}_k^c}\sum_{k=1}^{r} \underbrace{\left|\frac{1}{n}\tr{X^{(k)}_{j}\!\!}{\mathbf{I}\!-\!\frac{1}{n}X^{(k)}_{\mathcal{U}_k}\!\!\left(\!\frac{1}{n}\!\tr{X^{(k)}_{\mathcal{U}_k}\!\!}{X^{(k)}_{\mathcal{U}_k}}\!\right)^{\!\!-1} \!\!\!\!\!\left(X^{(k)}_{\mathcal{U}_k}\right)^T} w^{(k)}\right|}_{\mathcal{W}_j^{(k)}}\\
&\quad+\max_{j\in\bigcap_{k=1}^r\mathcal{U}_k^c} \sum_{k=1}^{r}\left|\frac{1}{n}\tr{X^{(k)}_{j}}{X^{(k)}_{\mathcal{U}_k}\left(\frac{1}{n}\tr{X^{(k)}_{\mathcal{U}_k}}{X^{(k)}_{\mathcal{U}_k}}\right)^{-1}} \tilde{z}^{(k)}_{\mathcal{U}_k}\right|\\
&\leq\max_{j\in\bigcap_{k=1}^r\mathcal{U}_k^c}\sum_{k=1}^r\left|\mathcal{W}_j^{(k)}\right| \\ &\quad+ \max_{j\in\bigcap_{k=1}^r\mathcal{U}_k^c} \sum_{k=1}^{r}\left\|\frac{1}{n}\tr{X^{(k)}_{j}}{X^{(k)}_{\mathcal{U}_k}\left(\frac{1}{n}\tr{X^{(k)}_{\mathcal{U}_k}}{X^{(k)}_{\mathcal{U}_k}}\right)^{-1}}\right\|_1\\ &\qquad\qquad\qquad\qquad\qquad\qquad\qquad\qquad\qquad\qquad\qquad\max_{j\in\bigcup_{k=1}^r\mathcal{U}_k}\left\|\tilde{z}^{(k)}_{j}\right\|_{1}\\
&\quad+\max_{j\in\bigcap_{k=1}^r\mathcal{U}_k^c} \sum_{k=1}^{r}\underbrace{\left|\frac{1}{n}\tr{R^{(k)}_{j}}{X^{(k)}_{\mathcal{U}_k}\left(\frac{1}{n}\tr{X^{(k)}_{\mathcal{U}_k}}{X^{(k)}_{\mathcal{U}_k}}\right)^{-1}} \tilde{z}^{(k)}_{\mathcal{U}_k}\right|}_{\mathcal{R}_j^{(k)}}\\
&\leq(1-\gamma_b)\lambda_b+\max_{j\in\bigcap_{k=1}^r\mathcal{U}_k^c}\sum_{k=1}^r\left|\mathcal{R}_j^{(k)}\right|+\max_{j\in\bigcap_{k=1}^r\mathcal{U}_k^c}\sum_{k=1}^r\left|\mathcal{W}_j^{(k)}\right|.\\
\end{aligned}
\nonumber
\end{equation}\normalsize
Let $\mathbf{v}\in\{-1,+1\}^r$ be a vector of signs such that $\sum_{k=1}^r\left|\mathcal{W}_j^{(k)}\right|=\sum_{k=1}^rv_k\mathcal{W}_j^{(k)}$. Then,
\begin{equation}
\text{Var}\left(\sum_{k=1}^{r}\left|\mathcal{W}_j^{(k)}\right|\right)=\text{Var}\left(\sum_{k=1}^{r}v_k\mathcal{W}_j^{(k)}\right)\leq\frac{2\sigma^2r}{n}.
\nonumber
\end{equation}
Using the union bound and previous discussion, we get
\begin{equation}
\begin{aligned}
&\mathbb{P}\left[\max_{j\in\bigcap_{k=1}^r\mathcal{U}_k^c}\sum_{k=1}^r\left|\mathcal{W}_j^{(k)}\right|\geq t\right]\\ &\qquad=\mathbb{P}\left[\max_{j\in\bigcap_{k=1}^r\mathcal{U}_k^c}\max_{\mathbf{v}\in\{-1,+1\}^r}\sum_{k=1}^r v_k\mathcal{W}_j^{(k)}\geq t\right]\\ &\qquad\leq 2\exp\left(-\frac{t^2n}{4\sigma^2r}+r\log(2)+\log(p)\right)\qquad\forall t\geq 0.\\
\end{aligned}
\nonumber
\end{equation}
We have
\small\begin{equation}
\begin{aligned}
\text{Var}\left(\sum_{k=1}^r\left|\mathcal{R}_j^{(k)}\right|\right)&=\text{Var}\left(\sum_{k=1}^r v_k\mathcal{R}_j^{(k)}\right)\\
&\leq\frac{\sum_{k=1}^r\left\|\tilde{z}_j^{(k)}\right\|_2^2}{nC_{min}}\leq\frac{rs\lambda_s^2}{nC_{min}}<\frac{rs\lambda_b^2}{nC_{min}}
\end{aligned}
\nonumber
\end{equation}\normalsize
and consequently by concentration of Gaussian variables,
\small\begin{equation}
\begin{aligned}
&\mathbb{P}\left[\max_{j\in\bigcap_{k=1}^r\mathcal{U}_k^c}\sum_{k=1}^K\left|\mathcal{R}_j^{(k)}\right|\geq t\right]\\ &\qquad=\mathbb{P}\left[\max_{j\in\bigcap_{k=1}^r\mathcal{U}_k^c}\max_{\mathbf{v}\in\{-1,+1\}^r}\sum_{k=1}^r v_k\mathcal{R}_j^{(k)}\geq t\right]\\ &\qquad\leq 2\exp\left(-\frac{t^2nC_{min}}{2rs\lambda_b^2} +r\log(2)+\log(p)\right)\qquad\forall t\geq 0.\\
\end{aligned}
\nonumber
\end{equation}\normalsize
Finally, we have
\footnotesize\begin{equation}
\begin{aligned}
&\!\mathbb{P}\!\left[\!\left\|P_{U_b^c}(\EstDual)\right\|_{\infty,1}\!\!\!<\!\lambda_b\right]\\ &\qquad \geq\mathbb{P}\left[\max_{j\in\bigcap_{k=1}^r\mathcal{U}_k^c}\sum_{k=1}^r\left|\mathcal{R}_j^{(k)}\right|+\max_{j\in\bigcap_{k=1}^r\mathcal{U}_k^c}\sum_{k=1}^r\left|\mathcal{W}_j^{(k)}\right|<\gamma_b\lambda_b\right]\\
&\qquad\geq\mathbb{P}\left[\max_{j\in\bigcap_{k=1}^r\mathcal{U}_k^c}\sum_{k=1}^r\left|\mathcal{R}_j^{(k)}\right|<t_0\right]\\ &\qquad\qquad\qquad\mathbb{P}\left[\max_{j\in\bigcap_{k=1}^r\mathcal{U}_k^c}\sum_{k=1}^r\left|\mathcal{W}_j^{(k)}\right|<\gamma_b\lambda_b-t_0\right]\\
&\qquad\geq\left(1-2\exp\left(-\frac{t_0^2nC_{min}}{2rs\lambda_b^2} +r\log(2)+\log(p)\right)\right)\\ &\qquad\qquad\qquad\left(1-2\exp\left(-\frac{(\gamma_b\lambda_b-t_0)^2n}{4\sigma^2r}+r\log(2)+\log(p)\right)\right).
\end{aligned}
\nonumber
\end{equation}\normalsize
This probability goes to $1$  for $t_0=\frac{\sqrt{Bs}\lambda_b}{\sqrt{Bs}\lambda_b+2\sigma\sqrt{C_{min}}}\gamma_b\lambda_b$ (the solution to $\frac{(\gamma_b\lambda_b-t_0)^2n}{4\sigma^2r}=\frac{t_0^2nC_{min}}{2rs\lambda_b^2}$), if
\footnotesize\begin{equation}
\lambda_b>\frac{\sqrt{4\sigma^2C_{min}r\Big(r\log(2)+\log(p)\Big)}}{\gamma_b\sqrt{nC_{min}}-\sqrt{Bsr\Big(r\log(2)+\log(p)\Big)}},
\nonumber
\end{equation}\normalsize
provided that $n>\frac{Bsr(r\log(2)+\log(p))}{\gamma_b^2C_{min}}$ as stated in the assumptions. Hence, with probability at least $1-c_1\exp\left(-c_2\left(r\log(2)+\log(p)\right)\right)$ the conditions of the Lemma~\ref{LemWitnessOptCond} are satisfied.\\
\end{proof}

\begin{lemma}
\footnotesize$$\displaystyle\mathbb{P}\left[\max_{1\leq k\leq r}\max_{1\leq j\leq p}\left\|X^{(k)}_j\right\|_2^2\leq 2n\right]\geq 1-\exp\left(\!-\!(1\!-\!\frac{\sqrt{3}}{2})n+\log(pr)\right).$$\normalsize
\label{l2bound}
\end{lemma}

\begin{proof}
Notice that $\|X^{(k)}_j\|_2^2$ is a $\chi^2$ random variable with $n$ degrees of freedom. According to \cite{LAUMAS}, we have
\begin{equation}
\mathbb{P}\left[\left\|X^{(k)}_j\right\|_2^2\geq t+(\sqrt{t}+\sqrt{n})^2\right]\leq\exp(-t)\qquad\forall t\geq 0.
\nonumber
\end{equation}
Letting $t=\left(\frac{\sqrt{3}-1}{2}\right)^2\!\!\!n$ and using the union bound, the result follows.
\end{proof}

\subsection{Proof of Theorem 3}
\label{appendix:proof_thm3} 

We will actually prove a more general theorem, from which Theorem 3 would follow as a corollary. Among shared features (with size $\alpha s$), we say a fraction $\tau$ has different magnitudes on $\bar{\Theta}$. Let $\tau_1$ be the fraction with larger magnitude on the first task 
and $\tau_2$ the fraction with larger magnitude on the second task (so that $\tau = \tau_1 + \tau_2$). Moreover, let $\frac{\lambda_b}{\lambda_s}=\kappa$ and 
\[f(\kappa) = f(\kappa,\tau,\alpha) = 2 - 2(1 - \tau) \alpha - 2 \tau \alpha \kappa + \left(\frac{1 + \tau}{2}\right) \alpha \kappa^2,\] 
and 
\[g(\kappa,\tau,\alpha) = \max\left(\frac{2\,f(\kappa)}{\kappa^2},f(\kappa)\right).\]

\setcounter{theorem}{3}
\setcounter{corollary}{3}
\begin{theorem}\label{Thm3Gen}
Under the assumptions of the Theorem 3, if \footnotesize $$\left|\left\{j\in\rowsupport(\Bst):\Big|\left|\Theta^{*(1)}_j\right|-\left|\Theta^{*(2)}_j\right|\Big| \leq c\lambda_s\right\}\right|=(1-\tau)\alpha s,$$ \normalsize then, the result of Theorem 3 holds for $$\theta(n,s,p,\alpha)=\frac{n}{g(\kappa,\tau,\alpha) \, s\log\left(p-(2-\alpha)s\right)}.$$
\end{theorem}

\begin{corollary}\label{TwoTaskwithGAP}
Under the assumptions of the Theorem~4, if the regularization penalties are set as $\kappa = \lambda_b/\lambda_s = \sqrt{2}$, then the result of Theorem 3 holds for $\theta(n,s,p,\alpha)=\frac{n}{\left(2-\alpha+(3-2\sqrt{2})\tau\alpha\right)s\log\left(p-(2-\alpha)s\right)}$.
\end{corollary}

\begin{proof}
Follows trivially by substituting $\kappa = \sqrt{2}$ in Theorem 4. Indeed, this setting of $\kappa$ can also be shown to minimize $g(\kappa,\tau,\alpha)$:	
\begin{align*}
&\min_{1<\kappa<2}\max\left(\frac{2\,f(\kappa)}{\kappa^2},f(\kappa)\right)\\
&=\min\left(\min_{1<\kappa\leq\sqrt{2}}\frac{2}{\kappa^2}\left(f(\kappa)\right),\min_{\sqrt{2}<\kappa<2}f(\kappa)\right)\\
&=2-\alpha+(3-2\sqrt{2}) \, \tau \, \alpha.
\end{align*}
\label{nbound}
\end{proof}

{\bf Proof of Theorem 3}: The proof follows from Corollary~\ref{TwoTaskwithGAP} by setting $\tau= 0$ and $\kappa=\sqrt{2}$. 

\vskip0.1in
We will now set out to prove Theorem~4. We will first need the following lemma.
\begin{lemma}
\label{LemGapLambda}
For any $j\in\rowsupport(\Bst)$, if $\left|S^{*(k)}_j\right|<c\lambda_s$ for some constant $c$ specified in the proof, then $\tilde{S}^{(k)}_j=0$ with probability $1-c_1\exp(-c_2n)$.
\label{DualVariableNorm}
\end{lemma}
\begin{proof}
Let $\check{S}$ be a matrix equal to $\tilde{S}$ except that $\check{S}^{(k)}_j=0$. Using the concentration of Gaussian random variables and optimality of $\tilde{S}$, we get
\footnotesize\begin{equation}
\begin{aligned}
&\mathbb{P}\left[\left|\tilde{S}^{(k)}_j\right|>0\right]\\ &\leq
\mathbb{P}\Bigg[2n\lambda_s\left|\tilde{S}^{(k)}_j\right|<\left\|y^{(k)}-X^{(k)}(\tilde{B}^{(k)}+\check{S}^{(k)})\right\|_2^2\\ &\qquad\qquad\qquad\qquad\qquad\qquad\qquad\qquad-\left\|y^{(k)}-X^{(k)}(\tilde{B}^{(k)}+\tilde{S}^{(k)})\right\|_2^2\Bigg]\\
&=\mathbb{P}\Biggr[2n\lambda_s<\Bigg(\frac{\left\|y^{(k)}-X^{(k)}(\tilde{B}^{(k)}+\check{S}^{(k)})\right\|_2^2}{\left\|\tilde{S}^{(k)}_jX^{(k)}_j\right\|_2}\\ &\qquad\qquad\qquad-\frac{\left\|y^{(k)}-X^{(k)}(\tilde{B}^{(k)}+\check{S}^{(k)})-\tilde{S}^{(k)}_jX^{(k)}_j\right\|_2^2}{\left\|\tilde{S}^{(k)}_jX^{(k)}_j\right\|_2}\Bigg)\left\|X^{(k)}_j\right\|_2\Biggr]\\
&\leq\mathbb{P}\left[2n\lambda_s<2\left\|X^{(k)}_j\right\|_2^2\left\|y^{(k)}-X^{(k)}(\tilde{B}^{(k)}+\check{S}^{(k)})\right\|_2\right]\\
&=\mathbb{P}\left[n\lambda_s<\left\|X^{(k)}_j\right\|_2^2\left\|X^{(k)}(B^{*(k)}+S^{*(k)}-\tilde{B}^{(k)}-\check{S}^{(k)})+w^{(k)}\right\|_2\right]\\
\end{aligned}
\nonumber
\end{equation}\normalsize

\noindent Using the $\ell_\infty$ bound on the error, for some constant $c$, we have
\footnotesize\begin{equation}
\begin{aligned}
\mathbb{P}\left[\left|\tilde{S}^{(k)}_j\right|>0\right]
&\leq\mathbb{P}\left[n\lambda_s<\frac{1}{c}\left|S^{*(k)}_j\right|\left\|X^{(k)}_j\right\|_2^2\right]\\
&=\mathbb{P}\left[\frac{c\lambda_s}{\left|S^{*(k)}_j\right|}n<\left\|X^{(k)}_j\right\|_2^2\right].
\end{aligned}
\nonumber
\end{equation}\normalsize

Notice that $\mathbb{E}[\|X^{(k)}_j\|_2^2]=n$. According to the concentration of $\chi^2$ random variables concentration theorems (see \cite{LAUMAS}), this probability vanishes exponentially fast in $n$ for $\left|\bar{S}^{(k)}_j\right|<c\lambda_s$.\\
\end{proof}

\subsection{Proof of Theorem 4}
We will now provide the proofs of different parts separately.\\

\begin{proof}
{\bf (Success):} Recall the constructed primal-dual pair $(\tilde{B},\tilde{S},\EstDual)$. It suffices to show that the dual variable $\EstDual$ satisfies the conditions (C3) and (C4) of Lemma~\ref{LemWitnessOptCond}. By Lemma~\ref{ThresholdCertificateLemma}, these conditions are satisfied with probability at least $1-c_1\exp(-c_2n)$ for some positive constants $c_1$ and $c_2$. Hence, $(\Bh,\Sh)=(\tilde{B},\tilde{S})$ is the unique optimal solution. The rest are direct consequences of Proposition~\ref{GaussianProp} for $C_{min}=1$ and $D_{max}=1$.\\

{\bf (Failure):} We prove this result by contradiction. Suppose there exist a solution to \eqref{EqnDirtyEstNoisy}, say $(\Bh,\Sh)$ such that $\sgn\left(\support(\Bh+\Sh)\right)=\sgn\left(\support(B^*+S^*)\right)$. By Lemma~\ref{signBSnoisy}, this is equivalent to having $\sgn\left(\support(\Bh)\right)=\sgn\left(\support(B^*)\right)$ and $\sgn\left(\support(\Sh)\right)=\sgn\left(\support(S^*)\right)$ and $\frac{\lambda_b}{\lambda_s} = \kappa$. 

Now, suppose $n < (1 - \nu) \max\left(\frac{2\,f(\kappa)}{\kappa^2},f(\kappa)\right) s \log (p - (2 - \alpha) s)$, for some $\nu > 0$. This entails that
 
either (i) $n < (1 - \nu) f(\kappa) s \log (p - (2 - \alpha) s)$, 

or (ii) $n < (1 - \nu) \left(\frac{2\,f(\kappa)}{\kappa^2}\right) s \log (p - (2 - \alpha) s)$. \\

{\bf Case (i):} We will show that with high probability, there exists $k$ for which, there exists $j\in\bigcap_{k=1}^r\mathcal{U}_k^c$ such that $\left|\tilde{Z}_j^{(k)}\right|>\lambda_s$. This is a contradiction to Lemma~\ref{LemOptNoisy}.

Using \eqref{EqnDualBarAssign} and conditioning on $(X^{(k)}_{\mathcal{U}_k},w^{(k)},\tilde{Z}^{(k)}_{\mathcal{U}_k})$, for all $j\in\bigcap_{k=1}^r\mathcal{U}_k^c$ we have that the random variables $\tilde{Z}_j^{(k)}$ are i.i.d. zero-mean Gaussian random variables with
\footnotesize\begin{equation}
\begin{aligned}
&\text{Var}\left(\tilde{Z}_j^{(k)}\right)\\ &=\Biggr\|\frac{1}{n}X^{(k)}_{\mathcal{U}_k}\left(\frac{1}{n}\tr{X^{(k)}_{\mathcal{U}_k}}{X^{(k)}_{\mathcal{U}_k}}\right)^{-1}\!\!\! \tilde{Z}^{(k)}_{\mathcal{U}_k}\\ &\qquad\qquad+ \frac{1}{n}\left(\mathbf{I}-\frac{1}{n}X^{(k)}_{\mathcal{U}_k}\left(\frac{1}{n}\tr{X^{(k)}_{\mathcal{U}_k}}{X^{(k)}_{\mathcal{U}_k}}\right)^{-1} \!\!\!\left(X^{(k)}_{\mathcal{U}_k}\right)^T\right) w^{(k)}\Biggr\|_2^2\\
&=\left\|\frac{1}{n}X^{(k)}_{\mathcal{U}_k}\left(\frac{1}{n}\tr{X^{(k)}_{\mathcal{U}_k}}{X^{(k)}_{\mathcal{U}_k}}\right)^{-1}\!\!\! \tilde{Z}^{(k)}_{\mathcal{U}_k}\right\|_2^2\\ &\qquad\qquad+\left\|\frac{1}{n}\left(\mathbf{I}-\frac{1}{n}X^{(k)}_{\mathcal{U}_k}\left(\frac{1}{n}\tr{X^{(k)}_{\mathcal{U}_k}}{X^{(k)}_{\mathcal{U}_k}}\right)^{-1} \!\!\!\left(X^{(k)}_{\mathcal{U}_k}\right)^T\right) w^{(k)}\right\|_2^2\\
\end{aligned}
\nonumber
\end{equation}\normalsize
The second equality holds by orthogonality of projections. We thus have
\footnotesize\begin{equation}
\begin{aligned}
&\text{Var}\left(\tilde{Z}_j^{(k)}\right)\\ &\geq\max\left(\lambda_{min}\left(\left(\frac{1}{n}\tr{X^{(k)}_{\mathcal{U}_k}}{X^{(k)}_{\mathcal{U}_k}}\right)^{-1}\right) \frac{\left\|\tilde{Z}^{(k)}_{\mathcal{U}_k}\right\|_2^2}{n}\right.\\
&\qquad\qquad\left.,\frac{\left\|\left(\mathbf{I}-\frac{1}{n}X^{(k)}_{\mathcal{U}_k}\left(\frac{1}{n}\tr{X^{(k)}_{\mathcal{U}_k}}{X^{(k)}_{\mathcal{U}_k}}\right)^{-1} \!\!\!\left(X^{(k)}_{\mathcal{U}_k}\right)^T\right) w^{(k)}\right\|_2^2}{n^2}\right)\\
&\geq \frac{\left\|\tilde{Z}^{(k)}_{\mathcal{U}_k}\right\|_2^2}{\left(\sqrt{n}+\sqrt{s}\right)^2}\\
\end{aligned}
\nonumber
\end{equation}\normalsize
The second inequality holds with probability at least $1-c_1\exp\left(-c_2\left(\sqrt{n}+\sqrt{s}\right)^2\right)$ as a result of \cite{DAVSZA} on the eigenvalues of Gaussian matrices. The third inequality holds with probability at least $1-c_3\exp(-c_4n)$ as a result of \cite{LAUMAS} on the magnitude of $\chi^2$ random variables. Considering $\tilde{B}+\tilde{S}$, assume that among shared features (with size $\alpha s$), a portion of $\tau_1$ has larger magnitude on the fist task and a portion of $\tau_2$ has larger magnitude on the second task (and consequently a portion of $1-\tau_1-\tau_2$ has equal magnitude on both tasks). Assuming $\lambda_b=\kappa\lambda_s$ for some $\kappa\in(1,2)$, we get 
\footnotesize\begin{equation}
\begin{aligned} 
\widetilde{\sigma}_1^2&:=\text{Var}\left(\tilde{Z}_j^{(1)}\right)\\ &= \frac{(1-\alpha)s\lambda_s^2+\tau_1\alpha s\lambda_s^2+\tau_2\alpha s(\lambda_b-\lambda_s)^2+(1-\tau_1-\tau_2)\alpha s\frac{\lambda_b^2}{4}}{(\sqrt{n}+\sqrt{s})^2}\\ 
&=:\frac{f_1(\kappa)s\lambda_s^2}{n \, \left(1+\sqrt{\frac{s}{n}}\right)^2}.
\end{aligned}
\nonumber
\end{equation}\normalsize
The first equality follows from the construction of the dual matrix and the fact that we have recovered the sign support correctly. The last strict inequality follows from the assumption that $\theta(n,p,s,\alpha)<1$. Similarly, we have
\footnotesize\begin{equation}
\begin{aligned}
\widetilde{\sigma}_2^2&:=\text{Var}\left(\tilde{Z}_j^{(2)}\right)\\ &>\frac{(1-\alpha)s\lambda_s^2+\tau_2\alpha s\lambda_s^2+\tau_1\alpha s(\lambda_b-\lambda_s)^2+(1-\tau_1-\tau_2)\alpha s\frac{\lambda_b^2}{4}}{n \, \left(1+\sqrt{\frac{s}{n}}\right)^2}\\
&=:\frac{f_2(\kappa)s\lambda_s^2}{n \, \left(1+\sqrt{\frac{s}{n}}\right)^2}.
\end{aligned}
\nonumber
\end{equation}\normalsize

Given these lower bounds on the variance, by results on Gaussian maxima (see \cite{DAVSZA}), for any $\delta > 0$, with high probability, 
\begin{equation}
\begin{aligned}
&\max_{1\leq k\leq r}\max_{j\in\bigcup_{k=1}^r\mathcal{U}_k}\left|\tilde{Z}_j^{(k)}\right|\\ &\qquad\qquad\geq
(1 -\delta) \sqrt{(\widetilde{\sigma}_1^2+\widetilde{\sigma}_2^2)\log\left(r\Big(p-(2-\alpha)s\Big)\right)}.
\end{aligned}
\nonumber
\end{equation}
This in turn can be bound as
\footnotesize\begin{equation}
\begin{aligned}
&(1 -\delta) \, (\widetilde{\sigma}_1^2+\widetilde{\sigma}_2^2)\log\left(r\Big(p-(2-\alpha)s\Big)\right)\\
&\qquad\qquad\geq(1-\delta)\frac{\left(f_1(\kappa)+f_2(\kappa)\right)\,s\;\log\left(r\Big(p-(2-\alpha)s\Big)\right)}{n \, \left(1+\sqrt{\frac{s}{n}}\right)^2}\lambda_s^2.\\
&\qquad\qquad\geq(1-\delta)\frac{f(\kappa)\,s\;\log\left(r\Big(p-(2-\alpha)s\Big)\right)}{n \, \left(1+\sqrt{\frac{s}{n}}\right)^2}\lambda_s^2.\\
\end{aligned}
\nonumber
\end{equation}\normalsize
Consider two cases:
\begin{enumerate}
\item $\frac{s}{n}=\Omega(1)$: In this case, we have $s > c n$ for some constant $c > 0$. Then,
\small\begin{align*}
	& (1-\delta)\frac{\left(f(\kappa)\right)\,s\;\log\left(r\Big(p-(2-\alpha)s\Big)\right)}{n \, \left(1+\sqrt{\frac{s}{n}}\right)^2}\lambda_s^2\\
	& \qquad\qquad=  (1-\delta)\frac{\left(f(\kappa)\right)\,(s/n)\;\log\left(r\Big(p-(2-\alpha)s\Big)\right)}{\left(1+\sqrt{s/n}\right)^2}\lambda_s^2\\
	& \qquad\qquad> c' f(\kappa)\;\log\left(r\Big(p-(2-\alpha)s\Big)\right) \, \lambda_s^2\\
	& \qquad\qquad> (1 + \epsilon) \lambda_s^2,
\end{align*}\normalsize
for any fixed $\epsilon > 0$, as $p \rightarrow \infty$.\\

\item $\frac{s}{n}\rightarrow 0$: In this case, we have $s/n = o(1)$. Here we will use that the sample size scales as $n < (1 - \nu) \left(f(\kappa)\right) s \log (p - (2 - \alpha) s)$.
\small\begin{align*}
& (1-\delta)\frac{\left(f(\kappa)\right)\,s\;\log\left(r\Big(p-(2-\alpha)s\Big)\right)}{n \, \left(1+\sqrt{\frac{s}{n}}\right)^2}\lambda_s^2\\
&\qquad\qquad\ge \frac{(1 - \delta) (1 - o(1))}{ 1- \nu} \lambda_s^2\\
&\qquad\qquad> (1 + \epsilon) \lambda_{s}^2,
\end{align*}
for some $\epsilon > 0$ by taking $\delta$ small enough.
\end{enumerate}\normalsize

Thus with high probability, $\exists k \exists j\in\bigcap_{k=1}^r\mathcal{U}_k^c$ such that $\left|\tilde{Z}_j^{(k)}\right|>\lambda_s$. This is a contradiction to Lemma~\ref{LemOptNoisy}.\\

\vspace{0.8cm}
{\bf Case (ii):} We need to show that with high probability, there exist a row that violates the sub-gradient condition of $\ell_\infty$-norm: $\exists j\in\bigcap_{k=1}^r\mathcal{U}_k^c$ such that $\left\|\tilde{Z}_j^{(k)}\right\|_1>\lambda_b$. This is a contradiction to Lemma~\ref{LemOptNoisy}.\\

Following the same proof technique, notice that $\sum_{k=1}^r\tilde{Z}_j^{(k)}$ is a zero-mean Gaussian random variable with $\text{Var}\left(\sum_{k=1}^r\tilde{Z}_j^{(k)}\right)\geq r(\widetilde{\sigma}_1^2+\widetilde{\sigma}_2^2)$. Thus, with high probability
\begin{equation} \max_{j\in\bigcap_{k=1}^r\mathcal{U}_k^c}\left\|\tilde{Z}_j^{(k)}\right\|_1\geq(1-\delta)\sqrt{r(\widetilde{\sigma}_1^2+\widetilde{\sigma}_2^2)\log\Big(p-(2-\alpha)s\Big)}.
\nonumber
\end{equation}
Following the same line of argument for this case, yields the required bound $\left\|\tilde{Z}_j^{(k)}\right\|_1> (1 + \epsilon)\lambda_b$.\\

This concludes the proof of the theorem.\\
\end{proof}

\begin{lemma}
Under assumptions of Theorem 3, the conditions (C3) and (C4) in Lemma~\ref{LemWitnessOptCond} hold with probability at least $1-c_1\exp(-c_2n)$ for some positive constants $c_1$ and $c_2$.
\label{ThresholdCertificateLemma}
\end{lemma}

\begin{proof}
First, we need to bound the projection of $\EstDual$ into the space $U_s^c$. Notice that
\footnotesize\begin{equation}
\left|\left(P_{U_s^c}(\EstDual)\right)_j^{(k)}\right|=\left\{
\begin{aligned}
&\frac{\lambda_b-\lambda_s\|\tilde{S}_j\|_0}{\left|M_{j}(\tilde{B})\right|-\|\tilde{S}_j\|_0}\\ &\qquad j\in\rowsupport(\tilde{B})\quad\&\quad (j,k)\notin\support(\tilde{S})\\
&\left|\tilde{Z}^{(k)}_{j}\right|\qquad\qquad j\in\bigcap_{k=1}^r\mathcal{U}_{k}^c\\
&0\qquad\qquad\qquad\qquad\qquad\text{ow}.\\
\end{aligned}\right..
\nonumber
\end{equation}\normalsize
By our assumption on the penalty regularizer coefficients, we have $\frac{\lambda_b-\lambda_s\|\tilde{S}_j\|_0}{\left|M_{j}^{\pm}(\tilde{B})\right|-\|\tilde{S}_j\|_0}<\lambda_s$. Moreover, we have
\footnotesize\begin{equation}
\begin{aligned}
&\left|\tilde{Z}^{(k)}_{j}\right|\\ &\leq\!\!\!\!\!\max_{j\in\bigcap_{k=1}^r\mathcal{U}_k^c} \underbrace{\left|\frac{1}{n}\tr{X^{(k)}_{j}\!\!}{\mathbf{I}\! -\!\frac{1}{n}X^{(k)}_{\mathcal{U}_k}\!\left(\!\frac{1}{n}\!\tr{X^{(k)}_{\mathcal{U}_k}}{X^{(k)}_{\mathcal{U}_k}}\right)^{\!\!-1} \!\!\!\!\left(X^{(k)}_{\mathcal{U}_k}\right)^T}\! w^{(k)}\right|}_{\mathcal{W}_j^{(k)}}\\
&\qquad+\max_{j\in\bigcap_{k=1}^r\mathcal{U}_k^c} \underbrace{\left|\frac{1}{n}\tr{X^{(k)}_{j}}{X^{(k)}_{\mathcal{U}_k}\left(\frac{1}{n}\tr{X^{(k)}_{\mathcal{U}_k}}{X^{(k)}_{\mathcal{U}_k}}\right)^{-1}} \tilde{Z}^{(k)}_{\mathcal{U}_k}\right|}_{\mathcal{Z}_j^{(k)}}\\ &\triangleq\max_{j\in\bigcap_{k=1}^r\mathcal{U}_k^c}\left|\mathcal{Z}_j^{(k)}\right|+\max_{j\in\bigcap_{k=1}^r\mathcal{U}_k^c}\left|\mathcal{W}_j^{(k)}\right|.
\end{aligned}
\nonumber
\end{equation}\normalsize
By Lemma~\ref{l2bound}, if $n\geq\frac{2}{2-\sqrt{3}}\log(pK)$ then with high probability $\left\|X_j^{(k)}\right\|_2^2\leq 2n$ and hence $\text{Var}\left(\mathcal{W}_j^{(k)}\right)\leq\frac{2\sigma^2}{n}$. Notice that $\mathbb{E}\left[\left\|X_j^{(k)}\right\|_2^2\right]=n$ and we added the factor of $2$ arbitrarily to use the concentration theorems. Using the concentration results for the zero-mean Gaussian random variable $\mathcal{W}_j^{(k)}$ and using the union bound, for all $t>0$, we get
\footnotesize\begin{equation}
\mathbb{P}\left[\max_{j\in\bigcap_{k=1}^r\mathcal{U}_k^c}\left|\mathcal{W}_j^{(k)}\right|\geq t\right]\leq 2\exp\left(-\frac{t^2n}{4\sigma^2}+\log\big(p-(2-\alpha)s\big)\right).\\
\nonumber
\end{equation}\normalsize

\noindent Conditioning on $\left(X^{(k)}_{\mathcal{U}_k},w^{(k)},\tilde{Z}^{(k)}\right)$'s, we have that $\mathcal{Z}_j^{(k)}$ is a zero-mean Gaussian random variable with
\small\begin{equation}
\text{Var}\left(\mathcal{Z}_j^{(k)}\right)\leq\frac{1}{n}\lambda_{max}\left(\left(\frac{1}{n}\tr{X^{(k)}_{\mathcal{U}_k}}{X^{(k)}_{\mathcal{U}_k}}\right)^{-1}\right)\left\|\tilde{Z}_{\mathcal{U}_k}^{(k)}\right\|_2^2.
\nonumber
\end{equation}\normalsize
According to the result of \cite{DAVSZA} on singular values of Gaussian matrices, for the matrix $X^{(k)}_{\mathcal{U}_k}$, for all $\delta>0$, we have
\footnotesize\begin{equation}
\mathbb{P}\left[\sigma_{min}\left(X^{(k)}_{\mathcal{U}_k}\right)\leq \left(1-\delta\right)\left(\sqrt{n}-\sqrt{s}\right)\right]\leq\exp\left(-\frac{\delta^2\left(\sqrt{n}-\sqrt{s}\right)^2}{2}\right),
\nonumber
\end{equation}\normalsize
and since \footnotesize$\lambda_{max}\left(\left(\tr{X^{(k)}_{\mathcal{U}_k}}{X^{(k)}_{\mathcal{U}_k}}\right)^{-1}\right)=\sigma_{min}\left(X^{(k)}_{\mathcal{U}_k}\right)^{-2}\,$\normalsize, we get
\footnotesize\begin{equation}
\begin{aligned}
&\mathbb{P}\left[\lambda_{max}\left(\left(\frac{1}{n}\tr{X^{(k)}_{\mathcal{U}_k}}{X^{(k)}_{\mathcal{U}_k}}\right)^{-1}\right)\geq \frac{\left(1+\delta\right)}{\left(1-\sqrt{\frac{s}{n}}\right)^2}\right]\\ &\qquad\qquad\qquad\qquad\leq\exp\left(-\frac{\left(\sqrt{\delta+1}-1\right)^2\left(\sqrt{n}-\sqrt{s}\right)^2}{2(1+\delta)}\right).
\end{aligned}
\nonumber
\end{equation}\normalsize

According to Lemma~\ref{DualVariableNorm}, if $\left|\left|\Theta^{*(1)}_j\right|-\left|\Theta^{*(2)}_j\right|\right| = o(\lambda_s)$, then with high probability $\tilde{S}_{j} = 0$, so that 
$|\tilde{\Theta}_{j}^{(1)}| = |\tilde{\Theta}_{j}^{(2)}|$. Thus, among shared features (with size $\alpha s$), a fraction $\tau$ have differing magnitudes on $\tilde{\Theta}$. Let $\tau_1$ be the fraction with larger magnitude on the first task 
and $\tau_2$ the fraction with larger magnitude on the second task (so that $\tau = \tau_1 + \tau_2$). Then, with high probability, recalling that $\lambda_b=\kappa\lambda_s$ for some $1<\kappa<2$, we get
\footnotesize\begin{equation}
\begin{aligned}
\text{Var}\left(\mathcal{Z}_j^{(1)}\right)&\leq\frac{\left\|\tilde{Z}_{\mathcal{U}_1}^{(1)}\right\|_2^2}{\left(\sqrt{n}-\sqrt{s}\right)^2}\\ &=\frac{(1-\alpha)s\lambda_s^2+\tau_1\alpha s\lambda_s^2+\tau_2\alpha s(\lambda_b-\lambda_s)^2+(1-\tau_1-\tau_2)\alpha s\frac{\lambda_b^2}{4}}{\left(\sqrt{n}-\sqrt{s}\right)^2}\\ &=\frac{\left(1-(1-\tau_1-\tau_2)\alpha-2\tau_2\alpha\kappa+\left(\tau_2+\frac{1-\tau_1-\tau_2}{4}\right)\alpha\kappa^2\right)s\lambda_s^2}{\left(\sqrt{n}-\sqrt{s}\right)^2}\\
&\triangleq\frac{f_1(\kappa)s\lambda_s^2}{\left(\sqrt{n}-\sqrt{s}\right)^2}.\\
\end{aligned}
\nonumber
\end{equation}\normalsize
Similarly,
\footnotesize\begin{equation}
\begin{aligned}
\text{Var}\left(\mathcal{Z}_j^{(2)}\right)&\leq\frac{\left\|\tilde{Z}_{\mathcal{U}_2}^{(2)}\right\|_2^2}{\left(\sqrt{n}-\sqrt{s}\right)^2}\\
&=\frac{\left(1-(1-\tau_1-\tau_2)\alpha-2\tau_1\alpha\kappa+\left(\tau_1+\frac{1-\tau_1-\tau_2}{4}\right)\alpha\kappa^2\right)s\lambda_s^2}{\left(\sqrt{n}-\sqrt{s}\right)^2}\\
&\triangleq\frac{f_2(\kappa)s\lambda_s^2}{\left(\sqrt{n}-\sqrt{s}\right)^2}.\\
\end{aligned}
\nonumber
\end{equation}\normalsize

\noindent By concentration of Gaussian random variables, we have
\footnotesize\begin{equation}
\begin{aligned}
&\mathbb{P}\left[\max_{j\in\bigcap_{k=1}^r\mathcal{U}_k^c}\left|\mathcal{Z}_j^{(k)}\right|\geq t\right]\\ &\qquad\qquad\leq 2\exp\left(-\frac{t^2\left(\sqrt{n}-\sqrt{s}\right)^2} {2f_k(\kappa)s\lambda_s^2}+\log\big(p-(1-\alpha)s\big)\right)\qquad\forall t\geq 0.\\
\end{aligned}
\nonumber
\end{equation}\normalsize
Using these bounds, we get
\footnotesize\begin{equation}
\begin{aligned}
&\mathbb{P}\!\left[\!\left\|P_{U_s^c}(\EstDual)\right\|_{\infty,\infty}\!\!\!<\!\lambda_s\right]\\ &\geq\mathbb{P}\left[\max_{j\in\bigcap_{k=1}^r\mathcal{U}_k^c}\left|\mathcal{Z}_j^{(k)}\right|+\max_{j\in\bigcap_{k=1}^r\mathcal{U}_k^c}\left|\mathcal{W}_j^{(k)}\right|<\lambda_s\qquad\forall \,1\leq k\leq K\right]\\
&\geq\mathbb{P}\left[\max_{j\in\bigcap_{k=1}^r\mathcal{U}_k^c}\left|\mathcal{Z}_j^{(k)}\right|<t_0\quad\forall\,1\leq k\leq r\right]\\ &\qquad\qquad\mathbb{P}\left[\max_{j\in\bigcap_{k=1}^r\mathcal{U}_k^c}\left|\mathcal{W}_j^{(k)}\right|<\lambda_s-t_0\quad\forall\,1\leq k\leq r\right]\\
&\geq\left(1-2\exp\left(-\frac{t_0^2\left(\sqrt{n}-\sqrt{s}\right)^2} {\left(f_1(\kappa)+f_2(\kappa)\right)s\lambda_s^2}+\log\big(p-(2-\alpha)s\big)+\log(r)\right)\right)\\ &\qquad\qquad\left(1-2\exp\left(-\frac{(\lambda_s-t_0)^2n}{4\sigma^2}+\log\big(p-(2-\alpha)s\big)+\log(r)\right)\right).
\end{aligned}
\nonumber
\end{equation}\normalsize
This probability goes to $1$ for \footnotesize$$t_0=\frac{\sqrt{\left(f_1(\kappa)+f_2(\kappa)\right)ns}\lambda_s}{\sqrt{\left(f_1(\kappa)+f_2(\kappa)\right)ns}\lambda_s+2\sigma(\sqrt{n}-\sqrt{s})}\lambda_s$$\normalsize  (the solution to $\frac{t_0^2\left(\sqrt{n}-\sqrt{s}\right)^2} {\left(f_1(\kappa)+f_2(\kappa)\right)s\lambda_s^2}=\frac{(\lambda_s-t_0)^2n}{4\sigma^2}$), if
\footnotesize\begin{equation}
\lambda_s>\frac{\sqrt{4\sigma^2\left(1-\sqrt{\frac{s}{n}}\right)^2\Big(\log(r)+\log\big(p-(2-\alpha)s\big)\Big)}} {\sqrt{n}-\left(\sqrt{s}+\sqrt{\left(f_1(\kappa)+f_2(\kappa)\right)s\Big(\log(r)+\log\big(p-(2-\alpha)s\big)\Big)}\right)}\\
\nonumber
\end{equation}\normalsize
provided that (substituting $r=2$), 
\footnotesize\begin{equation}
\begin{aligned}
n&>\left(f_1(\kappa)+f_2(\kappa)\right)s\log\Big(p-(2-\alpha)s\Big)\\
&\qquad+\Bigg(1+\left(f_1(\kappa)+f_2(\kappa)\right)\log(2)\\ &\qquad\qquad+2\sqrt{\left(f_1(\kappa)+f_2(\kappa)\right)\left(\log(2)+\log\Big(p-(2-\alpha)s\Big)\right)}\Bigg)s.
\end{aligned}
\nonumber
\end{equation}\normalsize
Since $f_1(\kappa)+f_2(\kappa)=f(\kappa)$ by definition, for large enough $p$ with $\frac{s}{p}=\mathbf{o}(1)$, we require
\begin{equation}
n>f(\kappa)s\log\Big(p-(2-\alpha)s\Big).
\label{eq:n_lim_sparse}
\end{equation}
\vspace{0.5cm}

\noindent Next, we need to bound the projection of $\EstDual$ into the space $U_b^c$. Notice that
\footnotesize\begin{equation}
\sum_{k=1}^r\left|\left(P_{U_b^c}(\EstDual)\right)_j^{(k)}\right|=\left\{
\begin{aligned}
&\lambda_s\|\tilde{S}_j\|_0\qquad\quad&j\in\bigcup_{k=1}^r\mathcal{U}_k-\rowsupport(\Bst)\\
&\sum_{k=1}^{r}\left|\tilde{Z}^{(k)}_{j}\right|&j\in\bigcap_{k=1}^r\mathcal{U}_k^c\\
&0&\text{ow}
\end{aligned}\right..
\nonumber
\end{equation}\normalsize
We have $\lambda_s\|\tilde{S}_j\|_0\leq\lambda_s D(\Ss)<\lambda_b$ by our assumption on the ratio of penalty regularizer coefficients. For all $j\in\bigcap_{k=1}^r\mathcal{U}_k^c$, we have
\footnotesize\begin{equation}
\begin{aligned}
&\sum_{k=1}^{r}\left|\tilde{Z}^{(k)}_{j}\right|\\ &\leq\!\!\!\!\!\!\!\max_{j\in\bigcap_{k=1}^r\mathcal{U}_k^c}\sum_{k=1}^{r} \underbrace{\left|\!\frac{1}{n}\!\tr{\!X^{(k)}_{j}\!\!}{\mathbf{I}\! -\!\frac{1}{n}X^{(k)}_{\mathcal{U}_k}\!\!\left(\!\frac{1}{n}\!\tr{X^{(k)}_{\mathcal{U}_k}}{X^{(k)}_{\mathcal{U}_k}}\right)^{\!\!-1} \!\!\!\!\!\left(X^{(k)}_{\mathcal{U}_k}\right)^T\!\!}\! w^{(k)}\right|}_{\mathcal{W}_j^{(k)}}\\
&\qquad+\max_{j\in\bigcap_{k=1}^r\mathcal{U}_k^c}\sum_{k=1}^{r} \underbrace{\left|\frac{1}{n}\tr{X^{(k)}_{j}}{X^{(k)}_{\mathcal{U}_k}\left(\frac{1}{n}\tr{X^{(k)}_{\mathcal{U}_k}}{X^{(k)}_{\mathcal{U}_k}}\right)^{-1}} \tilde{Z}^{(k)}_{\mathcal{U}_k}\right|}_{\mathcal{Z}_j^{(k)}}\\ &=\max_{j\in\bigcap_{k=1}^r\mathcal{U}_k^c}\sum_{k=1}^r\left|\mathcal{Z}_j^{(k)}\right|+\max_{j\in\bigcap_{k=1}^r\mathcal{U}_k^c}\sum_{k=1}^r\left|\mathcal{W}_j^{(k)}\right|.\\
\end{aligned}
\nonumber
\end{equation}\normalsize
Let $\mathbf{v}\in\{-1,+1\}^r$ be a vector of signs such that $\sum_{k=1}^r\left|\mathcal{W}_j^{(k)}\right|=\sum_{k=1}^rv_k\mathcal{W}_j^{(k)}$. Thus,
\small\begin{equation}
\text{Var}\left(\sum_{k=1}^{r}\left|\mathcal{W}_j^{(k)}\right|\right)=\text{Var}\left(\sum_{k=1}^{r}v_k\mathcal{W}_j^{(k)}\right)\leq\frac{2\sigma^2r}{n}.
\nonumber
\end{equation}\normalsize
Using the union bound and previous discussion, for all $t>0$, we get
\footnotesize\begin{equation}
\begin{aligned}
&\mathbb{P}\left[\max_{j\in\bigcap_{k=1}^r\mathcal{U}_k^c}\sum_{k=1}^r\left|\mathcal{W}_j^{(k)}\right|\geq t\right]\\ &\qquad\qquad\qquad\qquad=\mathbb{P}\left[\max_{j\in\bigcap_{k=1}^r\mathcal{U}_k^c}\max_{\mathbf{v}\in\{-1,+1\}^r}\sum_{k=1}^r v_k\mathcal{W}_j^{(k)}\geq t\right]\\ &\qquad\qquad\qquad\qquad\leq 2\exp\left(-\frac{t^2n}{4\sigma^2r}+r\log(2)+\log\big(p-(2-\alpha)s\big)\right).\\
\end{aligned}
\nonumber
\end{equation}\normalsize
Also from the previous analysis, assuming $\lambda_b=\kappa\lambda_s$ for some $1<\kappa<2$, we get
\footnotesize\begin{equation}
\begin{aligned}
&\text{Var}\left(\sum_{k=1}^r\left|\mathcal{Z}_j^{(k)}\right|\right)=\text{Var}\left(\sum_{k=1}^r v_k\mathcal{Z}_j^{(k)}\right)\leq\frac{\sum_{k=1}^r\left\|\tilde{Z}_j^{(k)}\right\|_2^2}{\left(\sqrt{n}-\sqrt{s}\right)^2}\\ &=\frac{2(1-\alpha)s\lambda_s^2+(\tau_1+\tau_2)\alpha s\lambda_s^2+(\tau_1+\tau_2)\alpha s(\lambda_b-\lambda_s)^2+2(1-\tau_1-\tau_2)\alpha s\frac{\lambda_b^2}{4}}{\left(\sqrt{n}-\sqrt{s}\right)^2}\\
&=\frac{\frac{1}{\kappa^2}\left(f_1(\kappa)+f_2(\kappa)\right)s\lambda_b^2}{\left(\sqrt{n}-\sqrt{s}\right)^2}.
\end{aligned}
\nonumber
\end{equation}\normalsize
and consequently for all $t>0$,
\footnotesize\begin{equation}
\begin{aligned}
&\mathbb{P}\left[\max_{j\in\bigcap_{k=1}^r\mathcal{U}_k^c}\sum_{k=1}^r\left|\mathcal{Z}_j^{(k)}\right|\geq t\right]\\&=\mathbb{P}\left[\max_{j\in\bigcap_{k=1}^r\mathcal{U}_k^c}\max_{\mathbf{v}\in\{-1,+1\}^r}\sum_{k=1}^r v_k\mathcal{Z}_j^{(k)}\geq t\right]\\ &\leq 2\exp\left(-\frac{t^2\left(\sqrt{n}-\sqrt{s}\right)^2}{\frac{1}{\kappa^2}\left(f_1(\kappa)+f_2(\kappa)\right)s\lambda_b^2} +r\log(2)+\log\big(p-(2-\alpha)s\big)\right).\\
\end{aligned}
\nonumber
\end{equation}\normalsize
Finally, we have
\footnotesize\begin{equation}
\begin{aligned}
&\!\mathbb{P}\!\left[\!\left\|P_{U_b^c}(\EstDual)\right\|_{\infty,1}\!\!\!<\!\lambda_b\right]\\& \geq\mathbb{P}\left[\max_{j\in\bigcap_{k=1}^r\mathcal{U}_k^c}\sum_{k=1}^r\left|\mathcal{Z}_j^{(k)}\right|+\max_{j\in\bigcap_{k=1}^r\mathcal{U}_k^c}\sum_{k=1}^r\left|\mathcal{W}_j^{(k)}\right|<\lambda_b\right]\\
&\geq\mathbb{P}\left[\max_{j\in\bigcap_{k=1}^r\mathcal{U}_k^c}\sum_{k=1}^r\left|\mathcal{Z}_j^{(k)}\right|<t_0\right]\\ &\qquad\qquad\qquad\mathbb{P}\left[\max_{j\in\bigcap_{k=1}^r\mathcal{U}_k^c}\sum_{k=1}^r\left|\mathcal{W}_j^{(k)}\right|<\lambda_b-t_0\right]\\
&\geq\left(1-2\exp\left(-\frac{t_0^2\left(\sqrt{n}-\sqrt{s}\right)^2}{\frac{1}{\kappa^2}\left(f_1(\kappa)+f_2(\kappa)\right)s\lambda_b^2} +r\log(2)+\log\big(p-(2-\alpha)s\big)\right)\right)\\ &\qquad\qquad\qquad\left(1-2\exp\left(-\frac{(\lambda_b-t_0)^2n}{4\sigma^2r}+r\log(2)+\log\big(p-(2-\alpha)s\big)\right)\right).
\end{aligned}
\nonumber
\end{equation}\normalsize
This probability goes to $1$  for \footnotesize$$t_0=\frac{\sqrt{\frac{1}{\kappa^2}\left(f_1(\kappa)+f_2(\kappa)\right)ns}\lambda_b}{\sqrt{\frac{1}{\kappa^2}\left(f_1(\kappa)+f_2(\kappa)\right)ns}\lambda_b+2\sigma(\sqrt{n}-\sqrt{s})}\lambda_b$$\normalsize (the solution to $\frac{(\lambda_b-t_0)^2n}{4\sigma^2r}=\frac{t_0^2(\sqrt{n}-\sqrt{s})^2}{\frac{1}{\kappa^2}\left(f_1(\kappa)+f_2(\kappa)\right)s\lambda_b^2}$), if
\footnotesize\begin{equation}
\lambda_b>\frac{\sqrt{4\sigma^2\left(1-\sqrt{\frac{s}{n}}\right)^2r\Big(r\log(2)+\log\big(p-(2-\alpha)s\big)\Big)}} {\sqrt{n}-\left(\sqrt{s}+\sqrt{\frac{1}{\kappa^2}\left(f_1(\kappa)+f_2(\kappa)\right)sr\Big(r\log(2)+\log\big(p-(2-\alpha)s\big)\Big)}\right)}\\
\nonumber
\end{equation}\normalsize
provided that (substituting $r=2$),
\footnotesize\begin{equation}
\begin{aligned}
n&>\frac{2}{\kappa^2}\left(f_1(\kappa)+f_2(\kappa)\right)s\log\Big(p-(2-\alpha)s\Big)\\ &\qquad\qquad+\Bigg(1+\frac{2}{\kappa^2}\left(f_1(\kappa)+f_2(\kappa)\right)2\log(2)\\ &\qquad\qquad\qquad+2\sqrt{\frac{2}{\kappa^2}\left(f_1(\kappa)+f_2(\kappa)\right)\left(2\log(2)+\log\Big(p-(2-\alpha)s\Big)\right)}\Bigg)s.
\end{aligned}
\nonumber
\end{equation}\normalsize
For large enough $p$ with $\frac{s}{p}=\mathbf{o}(1)$, we require 
\begin{equation}
n>\frac{2}{\kappa^2}f(\kappa)s\log\Big(p-(2-\alpha)s\Big).
\nonumber
\end{equation}

\noindent Combining this result with \eqref{eq:n_lim_sparse}, the lemma follows.\\
\end{proof}

\newpage
\bibliographystyle{plainnat}
\bibliography{mtdm}

\appendices
\section{Deterministic Necessary Optimality Conditions}
\label{appendix:Opt_Cond}

In this appendix, we investigate \underline{deterministic} necessary conditions for the optimality of the solutions $(\Bh,\Sh)$ of the problem \eqref{EqnDirtyEstNoisy}. 

\subsection{Sub-differential of $\ell_1/\ell_\infty$ and $\ell_1/\ell_1$ Norms}
In this section we state the sub-differential characterization of the norms we used in out convex program. The results can be directly derived from the definition of sub-differential of a function.

\begin{lemma} [Sub-differential of $\ell_1/\ell_\infty$-Norm]
The matrix $\widetilde{Z}\in\mathbb{R}^{p\times r}$ belongs to the sub-differential of $\ell_1/\ell_\infty$-norm of matrix $\widetilde{B}$, denoted as $\widetilde{Z}\in\partial\left\|\widetilde{B}\right\|_{1,\infty}$ iff 
\begin{itemize}
\item [(i)] for all $j\in\rowsupport(\widetilde{B})$, we have $\tilde{z}_j^{(k)}=\left\{\begin{array}{cc} t_j^{(k)}\,\sgn\left(\tilde{b}^{(k)}_{j}\right) & k \in M_{j}(\widetilde{B})\\
0 & \text{ow}.
\end{array}\right.$, where, $t_j^{(k)}\geq 0$ and $\sum_{k=1}^rt_j^{(k)}=1$.
\item [(ii)] for all $j\notin\rowsupport(\widetilde{B})$, we have $\sum_{k=1}^r\left|\tilde{z}_j^{(k)}\right|\leq 1$.
\end{itemize}
\end{lemma}  

\begin{lemma} [Sub-differential of $\ell_1/\ell_1$-Norm]
The matrix $\widetilde{Z}\in\mathbb{R}^{p\times r}$ belongs to the sub-differential of $\ell_1/\ell_1$-norm of matrix $\widetilde{S}$, denoted as $\widetilde{Z}\in\partial\left\|\widetilde{S}\right\|_{1,1}$ iff 
\begin{itemize}
\item [(i)] for all $(j,k)\in\support(\widetilde{S})$, we have $\tilde{z}_j^{(k)}=\sgn\left(\tilde{s}^{(k)}_{j}\right)$.
\item [(ii)] for all $(j,k)\notin\support(\widetilde{S})$, we have $\left|\tilde{z}_j^{(k)}\right|\leq 1$.
\end{itemize}
\end{lemma}  

\subsection{Necessary Conditions}
The first lemma shows a necessary condition for any solution of the problem  \eqref{EqnDirtyEstNoisy}.

\begin{lemma}
If $(\Sh,\Bh)$ is a solution (uniqueness is \underline{NOT} required) of (\ref{EqnDirtyEstNoisy}) then the following properties hold
\begin{itemize}
\item [(P1)] $\sgn(\Bs)=\sgn(\2b)$ for all $(j,k)\in\support(\Sh)$ with $j\in\rowsupport(\Bh)$.
\item [(P2)] if $\frac{\lambda_b}{\lambda_s}$ is not an integer, $\frac{1}{D(\Sh)}>\frac{\lambda_s}{\lambda_b}>\frac{1}{M(\Bh)}$.
\item [(P3)] $\left|\2b\right|=\left\|\1b\right\|_{\infty}$ for all $(j,k)\in\support(\Sh)$.
\item [(P4)] if $\frac{\lambda_b}{\lambda_s}$ is not an integer, $\forall j\, \exists k$ such that $(j,k)\notin\support(\Sh)$ and $\left|\2b\right|=\left\|\1b\right\|_{\infty}$.
\end{itemize}
\label{signBSnoisy}
\end{lemma}

\begin{proof}
We provide the proof of each property separately.
\begin{itemize}
\item [(P1)] Suppose there exists $(j_0,k_0)\in\support(\Sh)$, such that $\sgn(\Bs)=-\sgn(\2b)$. Let $\check{B},\check{S}\in\mathbb{R}^{p\times r}$ be matrices equal to $\Bh,\Sh$ in all entries except at $(j_0,k_0)$. Consider the following two cases
\begin{enumerate}
\item $\left|\Cs+\3b\right|\leq\left\|\4b\right\|_{\infty}$: Let $\check{b}_{j_0}^{\left(k_0\right)}=\3b+\Cs$ and $\check{s}_{j_0}^{\left(k_0\right)}=0$. Notice that $(j_0,k_0)\notin\support(\check{S})$.
\item $\left|\Cs+\3b\right|>\left\|\4b\right\|_{\infty}$: Let $\check{b}_{j_0}^{\left(k_0\right)}=-\sgn\left(\3b\right)\left\|\4b\right\|_{\infty}$ and $\check{s}_{j_0}^{\left(k_0\right)}=\Cs+\3b-\check{b}_{j_0}^{\left(k_0\right)}$. Notice that $\sgn\left(\check{b}_{j_0}^{\left(k_0\right)}\right)=\sgn\left(\check{s}_{j_0}^{\left(k_0\right)}\right)$. \end{enumerate}
Since $\check{B}+\check{S}=\Bh+\Sh$ and $\|\check{b}_{j_0}\|_{\infty}\leq\|\4b\|_{\infty}$ and $\|\check{s}_{j_0}\|_1<\|\Ds\|_1$, it is a contradiction to the optimality of $(\Bh,\Sh)$.\\
\item [(P2)] We prove the result in two steps by establishing 1. $M(\Bh)>\left\lfloor\frac{\lambda_b}{\lambda_s}\right\rfloor$ and 2. $D(\Sh)<\left\lceil\frac{\lambda_b}{\lambda_s}\right\rceil$.
\begin{enumerate}
\item In contrary, suppose there exists a row $j_0\in\rowsupport(\Bh)$ such that $\left|M_{j_0}(\Bh)\right|\leq\left\lfloor\frac{\lambda_b}{\lambda_s}\right\rfloor$. Let $k^*$ be the index of the element whose magnitude is ranked $\left(\left\lfloor\frac{\lambda_b}{\lambda_s}\right\rfloor+1\right)$ among the element of the vector $\4b+\Ds$. Let $\check{B},\check{S}\in\mathbb{R}^{p\times r}$ be matrices equal to $\Bh,\Sh$ in all entries except on the row $j_0$ and
\footnotesize\begin{equation}
\5b=\left\{
\begin{aligned}
&\left|\6b+\Fs\right|\sgn\left(\5b\right)\\ &\qquad\qquad\qquad\qquad\left|\5b+\Es\right|\geq\left|\6b+\Fs\right|\\
&\5b+\Es\qquad\qquad\qquad\qquad\qquad\text{ow},
\end{aligned}\right.
\nonumber
\end{equation}\normalsize
and $\check{s}_{j_0}=\Ds+\4b-\check{b}_{j_0}$. Notice that $M(\check{B})>\left\lfloor\frac{\lambda_b}{\lambda_s}\right\rfloor$ and $\sgn\left(\check{s}_{j_0}^{(k)}\right)=\sgn\left(\check{b}_{j_0}^{(k)}\right)$ for all $(j_0,k)\in\support\left(\check{s}_{j_0}\right)$ since $\sgn\left(\Es\right)=\sgn\left(\5b\right)$ for all $(j_0,k)\in\support\left(\Sh_{j_0}\right)$ by (P1). Further, since $\check{S}+\check{B}=\Sh+\Bh$ and $\|\check{b}_{j_0}\|_{\infty}=\left|\6b\right|+\left|\Fs\right|$ and $\|\check{s}_{j_0}\|_1\leq\|\Ds\|_1+ \left\lfloor\frac{\lambda_b}{\lambda_s}\right\rfloor\left(\left\|\4b\right\|_{\infty}-\left|\check{b}_{j_0}^{\left(k^*\right)}\right|-\left|\check{s}_{j_0}^{\left(k^*\right)}\right|\right)$, this is a contradiction to the optimality of $(\Bh,\Sh)$ due to the fact that $\lambda_s\left\lfloor\frac{\lambda_b}{\lambda_s}\right\rfloor<\lambda_b$.\\

\item In contrary, suppose there exists a row $j_0\in\rowsupport(\Sh)$ such that $\left\|\Ds\right\|_0\geq\left\lceil\frac{\lambda_b}{\lambda_s}\right\rceil$. Let $k^*$ be the index of the element whose magnitude is ranked $\left\lceil\frac{\lambda_b}{\lambda_s}\right\rceil$ among the elements of the vector $\4b+\Ds$. Let $\check{B},\check{S}\in\mathbb{R}^{p\times r}$ be matrices respectively equal to $\Bh$ and $\Sh$ in all entries except on the row $j_0$ and
\footnotesize\begin{equation}
\5b=\left\{
\begin{aligned}
&\left|\6b+\Fs\right|\sgn\left(\5b\right)\\ &\qquad\qquad\qquad\qquad\left|\5b+\Es\right|\geq\left|\6b+\Fs\right|\\
&\5b+\Es\qquad\qquad\qquad\qquad\qquad\text{ow},
\end{aligned}\right.
\nonumber
\end{equation}\normalsize
and $\check{s}_{j_0}=\Ds+\4b-\check{b}_{j_0}$. Notice that $D(\check{S})<\left\lceil\frac{\lambda_b}{\lambda_s}\right\rceil$ and $\sgn\left(\check{s}_{j_0}^{(k)}\right)=\sgn\left(\check{b}_{j_0}^{(k)}\right)$ for all $(j_0,k)\in\support\left(\check{s}_{j_0}\right)$ since $\sgn\left(\Es\right)=\sgn\left(\5b\right)$ for all $(j_0,k)\in\support\left(\Ds\right)$. Since $\check{S}+\check{B}=\Sh+\Bh$ and $\|\check{b}_{j_0}\|_{\infty}=\left|\6b\right|+\left|\Fs\right|$ and $\|\check{s}_{j_0}\|_1\leq\|\Ds\|_1+ \left(\left\lceil\frac{\lambda_b}{\lambda_s}\right\rceil-1\right)\left(\left\|\4b\right\|_{\infty}-\left|\check{b}_{j_0}^{\left(k^*\right)}\right|-\left|\check{s}_{j_0}^{\left(k^*\right)}\right|\right)$, this is a contradiction to the optimality of $(\Bh,\Sh)$, due to the fact that $\lambda_s\left(\left\lceil\frac{\lambda_b}{\lambda_s}\right\rceil-1\right)<\lambda_s\left\lfloor\frac{\lambda_b}{\lambda_s}\right\rfloor<\lambda_b$.\\
\end{enumerate}

\item [(P3)] If $j\notin\rowsupport(\Bh)$ then the result is trivial. Suppose there exists $(j_0,k_0)\in\support(\Sh)$ with $j_0\in\rowsupport(\Sh)$ such that $\left|b_{j_0}^{\left(k_0\right)}\right|<\|\4b\|_{\infty}$. Let $\check{B},\check{S}\in\mathbb{R}^{p\times r}$ be matrices equal to $\Bh,\Sh$ in all entries except for the entry corresponding to the index $(j_0,k_0)$. Let $\check{b}_{j_0}^{\left(k_0\right)}=\left\|\4b\right\|_{\infty}\sgn\left(\3b\right)$ if $\left|\3b+\Cs\right|\geq\left\|b_{j_0}\right\|_{\infty}$ and $\check{b}_{j_0}^{\left(k_0\right)}=\3b+\Cs$ otherwise. Let $\check{s}_{j_0}^{\left(k_0\right)}=\Cs+\3b-\check{b}_{j_0}^{\left(k_0\right)}$.
Since $\check{B}+\check{S}=\Bh+\Sh$ and $\left\|\check{b}_{j_0}\right\|_{\infty}=\left\|\4b\right\|_{\infty}$ and $\left\|\check{s}_{j_0}\right\|_{1}<\left\|\Ds\right\|_{1}$, it is a contradiction to the optimality of $(\Bh,\Sh)$.\\

\item [(P4)] If $j\notin\rowsupport(\Bh)$ or $j\notin\rowsupport(\Sh)$ the result is trivial. Suppose there exists a row $j_0\in\rowsupport(\Bh)\cap\rowsupport(\Sh)$ such that the result does not hold for that. Let $k^*=\arg\max_{\{k:(j,k)\notin\support(\Sh)\}}\left|\2b\right|$. Let $\check{B},\check{S}\in\mathbb{R}^{p\times r}$ be matrices equal to $\Bh,\Sh$ in all entries except for the row $j_0$ and
\begin{equation}
\5b=\left\{
\begin{aligned}
&\left|\6b\right|\sgn\left(\5b\right)\qquad &(j_0,k)\in\support(\Sh)\\
&\5b&\text{ow},
\end{aligned}\right.
\nonumber
\end{equation}
and $\check{s}_{j_0}=\Ds+\4b-\check{b}_{j_0}$. Since $\check{B}+\check{S}=\Sh+\Bh$ and $\left\|\check{b}_{j_0}\right\|_{\infty}=\left|\6b\right|$ and by (P2) and (P3), $\left\|\check{s}_{j_0}\right\|_1\leq\left\|\Ds\right\|_1 +\left(\left\lceil\frac{\lambda_b}{\lambda_s}\right\rceil-1\right)\left(\left\|\4b\right\|_{\infty}-\left|\6b\right|\right)$, this is a contradiction to the optimality of $(\Bh,\Sh)$, due to the fact that $\lambda_s\left(\left\lceil\frac{\lambda_b}{\lambda_s}\right\rceil-1\right)<\lambda_s\left\lfloor\frac{\lambda_b}{\lambda_s}\right\rfloor<\lambda_b$.\\
\end{itemize}
This concludes the proof of the lemma.\\
\end{proof}

The next lemma shows why the assumption that the ratio of penalty regularizer parameters is crucial for our analysis. This is not a deterministic result, but since it is related to optimality conditions, we included this lemma in this appendix.

\begin{lemma}
If $(\Sh,\Bh)$ with $\Bh\neq\mathbf{0}$ is a solution to (\ref{EqnDirtyEstNoisy}) and $d=\frac{\lambda_b}{\lambda_s}$ is an integer then $(\Sh,\Bh)$ is not the unique solution.
\label{Lem:non_uniqueness}
\end{lemma}

\begin{proof}
In contrary, assume that $(\Sh,\Bh)$ is the unique solution. Take a non-zero row $\4b$ with $j_0\in\rowsupport(\Bh)$. If $\left|M_{j_0}(\Bh)\right|<d$, then let $\check{B},\check{S}\in\mathbb{R}^{p\times r}$ be two matrices equal to $\Bh,\Sh$ except on the row $j_0$ and let $\check{b}_{j_0}=\mathbf{0}$ and $\check{s}_{j_0}=\4b+\Ds$. Then, $(\check{B},\check{S})$ are \emph{strictly} better solutions than $(\Bh,\Sh)$. This contradicts the optimality of $(\Bh,\Sh)$. Hence, $\left|M_{j_0}(\Bh)\right|\geq d$. with similar argument we can conclude that $\left\|\Sh_{j_0}\right\|_0\leq d$.\\

If $\left\|\Sh_{j_0}\right\|_0=d$, then let $0<\delta\leq\min_{(j_0,k)\in\support(\Sh)}\left|\Es\right|$ and $\check{B}(\delta),\check{S}(\delta)\in\mathbb{R}^{p\times r}$ be two matrices equal to $\Bh,\Sh$ except for the entries indexed $(j_0,k)\in\support(\Sh)$ and let $\check{b}_{j_0}^{(k)}=\5b+\delta\sgn\left(\5b\right)$ and $\check{s}_{j_0}^{(k)}=\Es-\delta\sgn\left(\Es\right)$ for all $(j_0,k)\in\support(\Sh)$. Then, $(\check{B}(\delta),\check{S}(\delta))$ is another solution to (\ref{EqnDirtyEstNoisy}). This contradicts the uniqueness of $(\Bh,\Sh)$.

If $\left\|\Sh_{j_0}\right\|_0<d$, then using Lemma~\ref{signBSnoisy} and Equation~\ref{errmag}, we have
\footnotesize\begin{equation}
\begin{aligned}
&\mathbb{P}\left[\left|M_{j_0}(\Bh)\right|\geq d+1\right]\\ &=\sum_{i=1}^{r-d}\mathbb{P}\left[\left|M_{j_0}(\Bh)\right|=d+i\right]\\
&=\sum_{i=1}^{r-d}\mathbb{P}\Bigg[\exists k_1,\ldots,k_{i+1}\!\in \!M_{j_0}(\Bh)\quad\forall l=1,\ldots,i+1:\\ &\qquad\qquad\qquad\qquad\qquad\qquad\qquad\|\7b+\underbrace{\Gs}_{0}|=\left\|\4b\right\|_{\infty}\Bigg]\\
&=\sum_{i=1}^{r-d}\mathbb{P}\Bigg[\exists k_1,\ldots,k_{i+1}\!\in \!M_{j_0}(\Bh)\quad\forall l=1,\ldots,i+1:\\ 
&\qquad\qquad\qquad\qquad\qquad\qquad\qquad\left|\delpar_{j_0}^{\left(k_l\right)}\right|=\left|\bsjkl+\ssjkl\right|+\left\|\1b\right\|_{\infty}\Bigg]\\
&=\sum_{i=1}^{r-d}\mathbb{P}\Bigg[\exists k_1,\ldots,k_{i+1}\!\in \!M_{j_0}(\Bh)\quad\forall l,m=1,\ldots,i+1:\\ 
&\qquad\qquad\qquad\qquad\qquad\qquad\qquad\left|\delpar_{j_0}^{\left(k_l\right)}\right|=C_{k_l,k_m}+\left|\delpar_{j_0}^{\left(k_m\right)}\right|\Bigg]=0.\\
\end{aligned}
\nonumber
\end{equation}\normalsize
In above equation $C_{k_l,k_m}$ are some constants. The last conclusion follows from the fact that $\delpar_{j_0}^{(k_l)}$'s are  continuous Gaussian variables and the cardinality of this event is less than the cardinality of the space they lie in. Hence, $\left|M_{j_0}(\Bh)\right|=d$.

Let $0<\delta<\left\|b_{j_0}\right\|_{\infty}$ and $\check{B}(\delta),\check{S}(\delta)\in\mathbb{R}^{p\times r}$ be two matrices equal to $\Bh,\Sh$ except for the entries indexed $(j_0,k)$ for $k\in M_{j_0}(\Bh)$ and let $\check{b}_{j_0}^{(k)}=\5b-\delta$ and $\check{s}_{j_0}^{(k)}=\Es+\delta$ for all $k\in M_{j_0}(\Bh)$. Then, $(\check{B}(\delta),\check{S}(\delta))$ is another solution to (\ref{EqnDirtyEstNoisy}). This contradicts the uniqueness of $(\Bh,\Sh)$.\\
\end{proof}

Next lemma characterizes the optimal solution by introducing a dual variable $\hat{Z}$.\\

\begin{lemma}[Convex Optimality]\label{LemOptNoisy}
If $(\Bh,\Sh)$ is a solution of \eqref{EqnDirtyEstNoisy} then there exists a matrix $\hat{Z}\in\mathbb{R}^{p\times r}$, called \emph{dual variable}, such that $\hat{Z}\in \lambda_s \partial \|\Sh\|_{1,1}$ and $\hat{Z} \in \lambda_b \partial \|\Bh\|_{1,\infty}$ and for all $k = 1,\hdots,r$,
\begin{equation}\label{EqnConvOptCond}
    \frac{1}{n} \tr{X^{(k)}}{X^{(k)}} \left(\Hs+\8b\right) - \frac{1}{n} (X^{(k)})^T y^{(k)} + \hat{z}^{(k)} = 0.\\
\end{equation}
\end{lemma}

\begin{proof}
The proof follows from the standard first order optimality argument.
\end{proof}

\section{Coordinate Descent Algorithm}
\label{appendix:coordinate_descent}

We use the coordinate descendent algorithm described as follows. The algorithm takes the tuple $(X, Y, \lambda_s, \lambda_b, \epsilon, B, S)$ as input, and outputs $(\Bh, \Sh)$. Note that $X$ and $Y$ are given to this algorithm, while $B$ and $S$ are our initial guess or the warm start of the regression matrices. $\epsilon$ is the precision parameter which determines the stopping criterion.\\
\\
We update elements of the sparse matrix $S$ using the subroutine $UpdateS$, and update elements in the block sparse matrix $B$ using the subroutine $UpdateB$, respectively, until the regression matrices converge. The pseudocode is in Algorithm 1 to Algorithm 3.\\

\begin{algorithm}
\caption{Our Model Solver}
\begin{algorithmic}
\REQUIRE $X$, $Y$, $\lambda_b$, $\lambda_s$, $B$, $S$ and $\varepsilon$
\ENSURE $\Sh$ and $\Bh$
\STATE
\STATE \textbf{Initialization:}
\FOR {$j=1:p$}
	\FOR {$k=1:r$}
		\STATE $c_j^{(k)}\leftarrow\tr{X_j^{(k)}}{y^{(k)}}$
		\FOR {$i=1:p$}
			\STATE $d_{i,j}^{(k)}\leftarrow\tr{X_i^{(k)}}{X_j^{(k)}}$
		\ENDFOR
	\ENDFOR
\ENDFOR
\STATE
\STATE \textbf{Updating:}
\LOOP	
	\STATE $S\leftarrow UpdateS(c; d; \lambda_s; B; S)$
	\STATE $B\leftarrow UpdateB(c; d; \lambda_b; B; S)$
	\IF {Relative Update $<\epsilon$}
		\STATE BREAK
	\ENDIF
\ENDLOOP 
\STATE RETURN $\Bh=B,\; \Sh=S$
\end{algorithmic}
\end{algorithm}

\begin{algorithm}
\caption{UpdateB}
\begin{algorithmic}
\REQUIRE c, d, $\lambda_b$, $B$ and $S$
\ENSURE $B$
\STATE Update $B$ using the cyclic coordinate descent algorithm for $\ell_1/\ell_\infty$ while keeping $S$ unchanged.
\STATE
\FOR {$j=1:p$}
	\FOR {$k=1:r$}
		\STATE $\alpha_j^{(k)} \leftarrow c_j^{(k)}-\sum_{i\neq j} (b_i^{(k)}+s_i^{(k)})d_{i,j}^{(k)}-s_i^{(k)}d_{j,j}^{(k)}$
		\IF {$\sum_{k = 1}^r |\alpha ^{(k)} _j | \leq \lambda_b $}
			\STATE $b_j\leftarrow 0$
		\ELSE
			\STATE Sort $\alpha$ to be $|\alpha_j^{(k_1)}| \ge |\alpha_j^{(k_2)}|\geq\cdot\cdot\cdot\geq |\alpha_j^{(k_r)}|$
			\STATE $m^*=\arg \max _{1 \le m \le r} (\sum_{k = 1}^r |\alpha ^{(k_m)} _j | - \lambda_b)/m$
			\FOR {$i=1:r$}
				\IF {$i>m^*$}
					\STATE $b_j^{(k_i)}\leftarrow\alpha_j^{(k_i)}$
				\ELSE 
					\STATE $b^{(k_i)} _j\leftarrow\frac{\sgn(\alpha ^{(k_i )} _j )}{m^*}\left(\sum_{l = 1}^{m^*}|\alpha ^{(k_l )} _j | - \lambda_b\right)$
				\ENDIF
			\ENDFOR
		\ENDIF
	\ENDFOR
\ENDFOR
\STATE RETURN $B$
\end{algorithmic}
\end{algorithm}

\begin{algorithm}[t]
\caption{Update-S}
\begin{algorithmic}
\REQUIRE c, d, $\lambda_s$, $B$ and $S$
\ENSURE $S$
\STATE Update $S$ using the cyclic coordinate descent algorithm for LASSO while keeping $B$ unchanged.

\FOR {$j=1:p$}
	\FOR {$k=1:r$}
		\STATE $\alpha_j^{(k)} \leftarrow c_j^{(k)}-\sum_{i \neq j} (b_i^{(k)}+s_i^{(k)})d_{i,j}^{(k)}-s_i^{(k)}d_{j,j}^{(k)}$
		\IF {$|\alpha_j^{(k)}| \leq \lambda_s$}
			\STATE $s_j^{k}\leftarrow 0$
		\ELSE
			\STATE $s_j^{k}\leftarrow\alpha_j^{(k)}-\lambda_s\sgn(\alpha_j^{(k)})$
		\ENDIF
	\ENDFOR
\ENDFOR
\STATE RETURN $S$
\end{algorithmic}
\end{algorithm}

\subsection{Correctness of Algorithms}
In this algorithm, $B$ is the block sparse matrix and $S$ is the sparse matrix. We alternatively update $B$ and $S$ until they converge. When updating $S$, we cycle through each element of $S$ while holding all the other elements of $S$ and $B$ unchanged; When updating $B$, we update each block $B_j$ (the coefficient vector of the $j^{th}$ feature for $r$ tasks) as a whole, while keeping $S$ and other coefficient vector of $B$ fixed.\\
\\
For updating $B$, the subproblem is updating $B_j$
\begin{eqnarray}
\1b = \arg\min_{b_j} && \frac{1}{2} \sum_{k=1}^{r}  
\left\|r_j^{(k)} - b_j^{(k)}X_j^{(k)} \right\|_{2}^{2} + 
\lambda_b \|b_j\|_{\infty}.
\end{eqnarray}

If we take the partial residual vector $r_j^{(k)}=y^{(k)}-\sum\limits_{l \ne j} (b_l^{(k)}X_l^{(k)})-\sum_{l}(s_l^{(k)}X_l^{(k)})$, the correctness of this algorithm will directly follow from the correctness of coordinate descent algorithm of $\ell_1/\ell_{inf}$ in \cite{BlockCD}. With the same argument, the correctness of the Algorithm 3 can be proven.\\

\end{document}